\PassOptionsToPackage{prologue,dvipsnames,table}{xcolor} 
\documentclass[10pt,twocolumn,letterpaper]{article}

\usepackage{cvpr}              

\usepackage{tikz}
\usetikzlibrary{shapes.geometric}

\providecommand{\todo}[1]{{\color{red}#1}}

\usepackage{pgfplots}

\usepackage{amsmath,amsfonts,bm}

\def\eqref#1{equation~\ref{#1}}

\def\1{\bm{1}}

\DeclareMathAlphabet{\mathsfit}{\encodingdefault}{\sfdefault}{m}{sl}
\SetMathAlphabet{\mathsfit}{bold}{\encodingdefault}{\sfdefault}{bx}{n}

\usepackage{url}
\usepackage{booktabs}
\usepackage{adjustbox}
\usepackage{multicol}
\usepackage{multirow}
\usepackage{colortbl}
\usepackage{caption}
\usepackage{subcaption}
\usepackage{floatrow}
\usepackage{pgfplots}
\usepackage[accsupp]{axessibility}  
\usepackage[normalem]{ulem}
\usepackage{graphicx}
\usepackage{float}
\usepackage{derivative}
\usepackage{tabularx}
\usepackage{minted}
\usepackage{boldline}
\usepackage{algorithm}
\usepackage{algorithmic}
\usepackage{tikz}
\usepackage{wrapfig}
\usepackage{comicneue}
\usepackage{overpic}
\usepackage{natbib} 

\usepackage{amsmath}
\usepackage{amssymb}
\usepackage{amsthm}

\definecolor{cvprblue}{rgb}{0.21,0.49,0.74}
\usepackage[pagebackref,breaklinks,colorlinks,allcolors=cvprblue]{hyperref}

\newtheorem{theorem}{Theorem}
\newtheorem{lemma}{Lemma}

\newtheorem{corollary}{Corollary}

\providecommand{\todo}[1]{}
\newcommand\hide[1]{}
\newcommand\remove[1]{{}}

\newcommand{\replace}[2]{{#2}}
\newcommand\add[1]{{#1}}

\definecolor{vcacolor}{RGB}{123,50,210}

\definecolor{lbcolor}{RGB}{224, 135, 0}

\definecolor{gold}{RGB}{219, 198, 57}
\definecolor{silver}{RGB}{184, 183, 182}
\definecolor{bronze}{RGB}{191, 138, 57}

\newcolumntype{g}{>{\columncolor[HTML]{F0F0F0}}c}

\def\paperID{00000} 
\def\confName{CVPR}
\def\confYear{2027}

\title{Beyond Random Couplings: Contrastive Noise Alignment in Generative Flows}

\author{
\begin{tabular}{cccc}
Lennart Wittke\textsuperscript{1, 2} \textsuperscript{$\dagger$}&
Vinicius Azevedo\textsuperscript{2}
\end{tabular}
\\[1.2em]
\begin{tabular}{l}
    \textsuperscript{1}ETH Z\"{u}rich \hspace{25pt}
    \textsuperscript{2}Disney Research | Studios
\end{tabular}
}

\begin{document}

\twocolumn[{%
\begin{center}
    \renewcommand\twocolumn[1][]{#1}%
    \maketitle
    
    \small
    \resizebox{\textwidth}{!}{
    \begin{tikzpicture}[
        nodes={},
        dot/.style={circle, fill=gray!60, inner sep=1.2pt},
        hero/.style={circle, fill=black, inner sep=2pt},
        othernoise/.style={circle, fill=gray!80, inner sep=1.8pt},
        target/.style={star, star points=5, star point ratio=2, fill=blue!80, inner sep=1.5pt},
        unpairedtarget/.style={star, star points=5, star point ratio=2, fill=blue!40, inner sep=1.5pt}
    ]
    
    \pgfmathsetseed{42}

    \begin{scope}[shift={(1,0)}]
        \node[align=center] at (0.25, 2.5) {\textbf{A.} Independent Coupling\\$\pi(x_0, x_1) = \mu_0(x_0)\mu_1(x_1)$};
        
        \node at (0, 0.2) {\includegraphics[width=3.5cm]{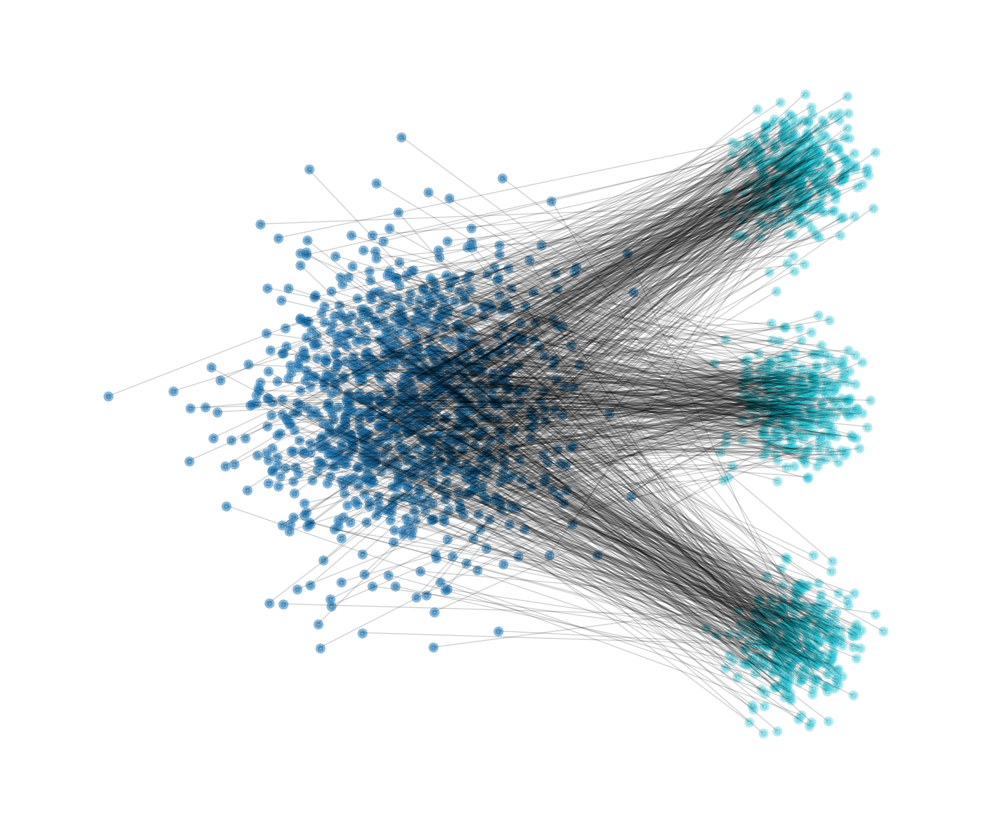}};
        \node[font=\sffamily\scriptsize, align=center] at (0, -1.6) {Unstructured Prior \\ (Crossed Paths)};
        
    \end{scope}

    \begin{scope}[shift={(6,0)}] 
        
        \node[align=center] at (3.75, 2.5) {\textbf{B.} Contrastive Noise Alignment \\ (CNA)};

        \node at (3.5, 0.2) {\includegraphics[width=3.8cm]{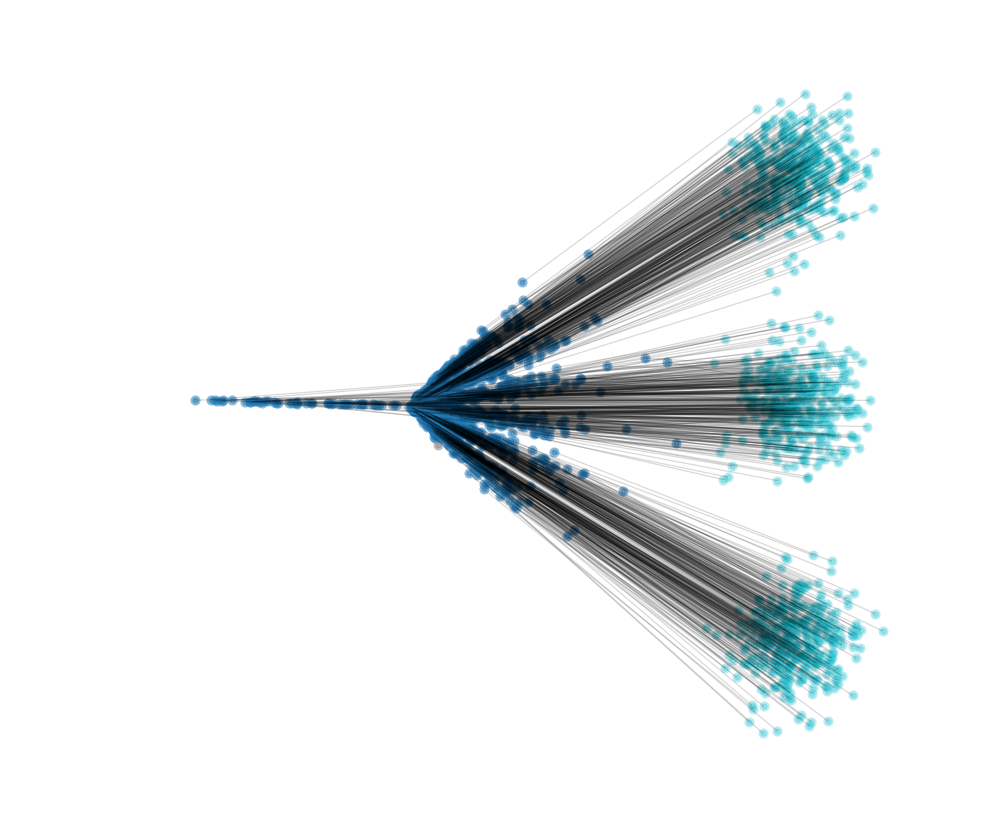}};
        \node[font=\sffamily\scriptsize, align=center] at (3.5, -1.6) {w/o Repulsion \\ (Prior Collapse)};

        \node at (7.5, 0.2) {\includegraphics[width=3.8cm]{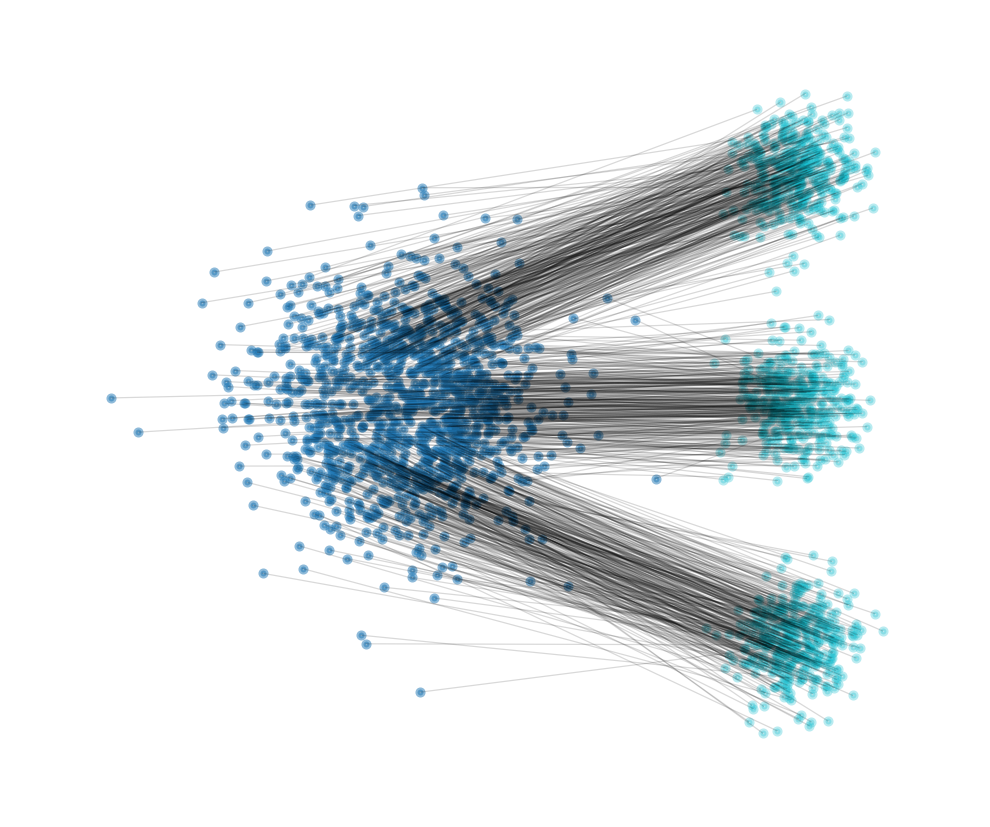}};
        \node[font=\sffamily\scriptsize, align=center] at (7.5, -1.6) {Full CNA \\ (Aligned Prior)};

        \begin{scope}[shift={(-0.2, -0.1)}, scale=0.98, transform shape]
            \draw[dashed, thick, gray!80] (0,0) circle (1.5cm);
            \node[anchor=north west, font=\fontfamily{ComicNeue-TLF}\selectfont\scriptsize, xshift=1pt] at (0,0) {Origin};
            \fill[black] (0,0) circle (1.5pt);
            
            \coordinate (Z) at (1.06, 1.06); 
            \node[hero] at (Z) {};
            \node[anchor=south, yshift=4pt] at (Z) {$\hat{z}_i$}; 
            
            \coordinate (Zj) at (-1.06, 1.06); 
            \node[othernoise] at (Zj) {};
            \node[anchor=south east, xshift=2pt] at (Zj) {$\hat{z}_j$};
            
            \coordinate (X) at (2.2, 1.8);
            \node[target] at (X) {};
            \node[anchor=south, yshift=2pt] at (X) {$x_i$};
            
            \coordinate (X_unpaired) at (-0.6, 2.0);
            \node[unpairedtarget] at (X_unpaired) {};
            \node[anchor=south] at (X_unpaired) {$x_j$};
            
            
            \draw[-stealth, thick, purple, shorten >=3pt] (Z) -- (0,0) 
                node[pos=0.5, anchor=north west, font=\fontfamily{ComicNeue-TLF}\selectfont\scriptsize, text=purple, xshift=-2pt, yshift=2pt] {$\mathcal{L}_{\text{norm}}$};
                
            \draw[-stealth, thick, blue, shorten >=3pt] (Z) -- (X) 
                node[pos=0.55, anchor=south east, font=\fontfamily{ComicNeue-TLF}\selectfont\scriptsize, text=blue, xshift=4pt, yshift=-2pt] {$\mathcal{L}_{\text{align}}^{+}$};
                
            \draw[dotted, thick, gray!50] (X_unpaired) -- (Z);
            \draw[-stealth, thick, dashed, blue, shorten >=2pt] (Z) -- (1.8, 0.6) 
                node[pos=0.8, anchor=south west, font=\fontfamily{ComicNeue-TLF}\selectfont\scriptsize, text=blue, xshift=-2pt, yshift=-1pt] {$\mathcal{L}_{\text{align}}^{-}$};
                
            \draw[stealth-stealth, thick, orange, shorten >=3pt, shorten <=3pt] (Z) -- (Zj) 
                node[pos=0.5, anchor=north, font=\fontfamily{ComicNeue-TLF}\selectfont\scriptsize, text=orange, yshift=-2pt] {$\mathcal{L}_{\text{entropy}}$};
        \end{scope}

    \end{scope}

    \end{tikzpicture}
    } 
    \vspace{0.3cm}
    \captionsetup{hypcap=false}
    
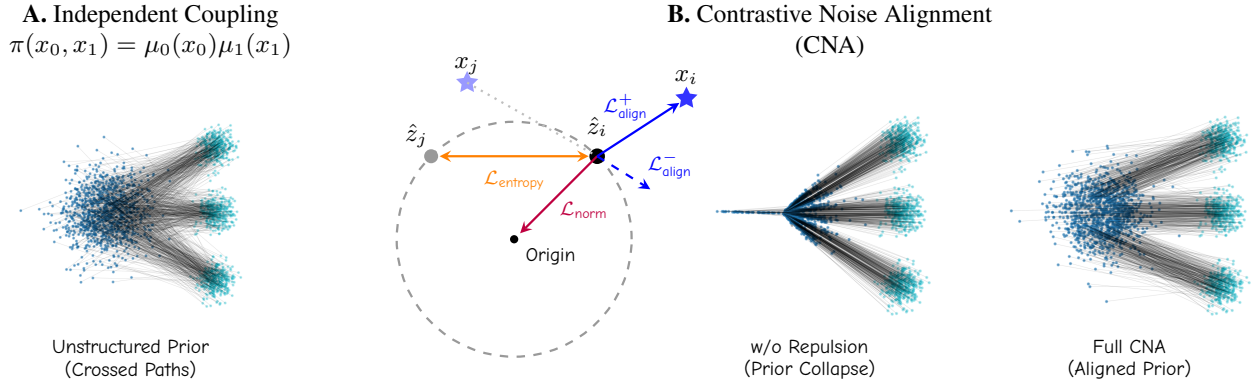
\captionof{figure}{\textbf{Illustration of Contrastive Noise Alignment (CNA)}. We model the noise batch as interacting particles. \textbf{A}: Independent Couplings lack alignment with the data, resulting in unstructured couplings and highly crossed paths. \textbf{B}: CNA balances three forces on each particle $z_i$: (1) \textit{semantic} alignment attracts it to its target $x_i$ (\textcolor{blue}{$\mathcal{L}_{\text{align}}^{+}$}) and repels it from unmatched targets $x_j$ (\textcolor{blue}{$\mathcal{L}_{\text{align}}^{-}$}); (2) angular repulsion between noise particles maximizes surface entropy (\textcolor{orange}{$\mathcal{L}_{\text{entropy}}$}); and (3) radial gravity bounds the space (\textcolor{purple}{$\mathcal{L}_{\text{norm}}$}). This yields structured couplings with significantly reduced path crossings while maintaining an approximately Gaussian distribution (i.e., a \textit{relaxed prior}).}
    \label{fig:cno_mechanics}
    \vspace{0.5cm}
\end{center}
}]

\begin{abstract}
Diffusion and flow-matching models are typically trained by corrupting data through independently sampled Gaussian noise. While simple and scalable, this forward process induces arbitrary data-noise couplings, forcing the network to learn high-curvature transports between unrelated endpoints. Existing optimal-transport methods reduce this burden by reassigning fixed noise samples to data, but the source noise distribution itself remains passive. To address this, we introduce Contrastive Noise Alignment (CNA), a training-time method that creates dynamic, contrastive couplings by optimizing the noise representations directly. By modeling the noise batch as an interacting particle system, CNA employs a cross-modal InfoNCE objective to align noise particles with their paired data targets. To prevent spatial collapse, this alignment is regularized using an angular entropy term and a radial norm penalty. We show theoretically that this equilibrium asymptotically preserves Gaussian structures, maintaining tractability during inference. Empirically, CNA improves the alignment between noise and data, reduces flow curvature, and provides better generation quality with fewer required sampling steps. For few-step, pixel-space generation (2-4 NFEs), CNA reduces FID by over 50\% compared to standard rectified flow, and by at least 24\% against Optimal Transport baselines.
\end{abstract}

\section{Introduction}
\label{sec:intro}
Modern diffusion and flow-matching models are commonly trained by coupling each data sample with an independently drawn Gaussian noise sample. This choice is attractive for obvious reasons: the Gaussian prior is analytically tractable, easy to sample, maximum-entropy under fixed first and second moments, and compatible with simple stochastic noising processes. Yet this convenience hides a geometric weakness. The coupling between data and noise is largely arbitrary: each image is paired with a randomly sampled endpoint, and the model must learn a denoising or velocity field whose regression target averages over many such unstructured pairings. At large scale this works remarkably well, but it also places the burden of discovering useful transport structure entirely on the network.

Recent work in flow matching makes this issue explicit. Minibatch optimal-transport couplings \citep{pooladian_multisample_2023, tong_improving_2024, davtyan_faster_2025} reduce the complexity of the learned flow by replacing independent noise-data pairings with geometrically shorter assignments. Semidiscrete formulations \citep{mousavi-hosseini_flow_2026} further improve scalability by exploiting the finite dataset structure, while normalizing-flow distillation \citep{berthelot_coupling_2026} uses a pretrained invertible model to provide stronger coupling supervision. These methods suggest a common principle: generative training improves when the coupling between data and noise is not left entirely to chance. However, existing approaches either rely on simplified geometric costs \citep{pooladian_multisample_2023, tong_improving_2024, davtyan_faster_2025}, expensive transport solvers \citep{mousavi-hosseini_flow_2026}, or auxiliary models whose own training depends on prior coupling assumptions \citep{berthelot_coupling_2026}.

In this work, we ask a complementary question: instead of selecting a Gaussian endpoint independently from the data sample, can we construct a data-dependent endpoint that is still statistically indistinguishable from Gaussian noise? Rather than pushing the geometric burden entirely onto downstream regression or discrete assignments, we propose optimizing the noise representations directly at training time to construct data-dependent endpoints that remain statistically indistinguishable from standard Gaussian noise.

Motivated by contrastive learning \cite{oord2018representation}, we introduce \textit{Contrastive Noise Alignment (CNA)}, which models the noise batch as an interacting particle system subject to a balance of attractive and repulsive forces. As we show theoretically, the combination of these forces asymptotically preserves standard multivariate Gaussian structure, maintaining prior tractability during inference.

We evaluate whether the proposed approach improves sample quality under matched compute and straightens generative trajectories. Our experiments compare CNA against standard independent coupling (I-CFM) and minibatch OT (OT-CFM) baselines across unconditional generation tasks on CIFAR-10 and ImageNet32. By shifting the burden of structure discovery into the prior itself, CNA provides a principled alternative to arbitrary random couplings.

In summary, our core contributions are:
\begin{itemize}
    \item We propose \textit{Contrastive Noise Alignment (CNA)}, a training-time method that dynamically aligns noise representations with data targets to construct structured forward couplings while asymptotically preserving the standard Gaussian prior.
    \item We show that training flow models with these optimized noise particles leads to fewer sampling steps required to generate high-quality samples.
\end{itemize}

\section{Related Work}
\label{sec:related_work}

\paragraph{Flow Matching and Optimal Transport.} Continuous normalizing flows and diffusion models have driven recent advances in generative modeling \citep{ho2020denoising, song2021score, lipman2023flow}. Flow matching simplifies training by regressing vector fields along linear probability paths \citep{liu2023flow}. Under standard independent couplings, noise and data are paired randomly, causing marginal trajectories to cross and complicating velocity field regression. To mitigate this, methods like Minibatch OT \citep{tong_improving_2024, pooladian_multisample_2023} and global alignments \citep{kong2025alignflow, mousavi-hosseini_flow_2026} solve semi-discrete optimal transport problems to minimize transport cost. Additionally, \citet{albergo2023stochastic} introduce data-dependent couplings within stochastic interpolants. While these approaches find optimal permutations, they treat the underlying Gaussian prior as a rigid, static set. In contrast, our method actively reshapes the continuous spatial topology of the noise prior, moving beyond rigid matching to full distributional optimization.

\paragraph{Few-Step Sampling.} Straightening generation trajectories inherently accelerates inference. Consequently, extensive literature explores model distillation \citep{song2023consistency, yin2024one, luo2025learning}, modified training objectives like Shortcut Models and \textit{MeanFlow} \citep{frans2025one, geng2026mean}, and instance-aware discretization solvers \citep{yuan2026fewstep}. Because these techniques modify the network's regression objective or rely on pre-trained weights, they leave the foundational noise-data geometry unchanged. Our proposed prior optimization is fully orthogonal to these methods; by establishing a structurally aligned prior during the initial training phase, we produce inherently straighter paths that can further ease the downstream burden on few-step samplers.

\paragraph{Optimized Source Distributions.} Recent work demonstrates that optimizing the initial source distribution can also yield straighter trajectories. \citet{nayal2026mixflow} propose \textit{MixFlow}, which reduces path
curvature by mixing conditioned and unconditional sources, while \citet{kim2026better} learn a prior that adapts to specific conditioning variables. Unlike these methods, which introduce parameterized prior networks or mixture models, we rely entirely on an implicit, non-parametric contrastive update without adding architectural overhead. 

\paragraph{Contrastive Objectives in Generative Modeling.} The InfoNCE objective \citep{oord2018representation} is a staple of representation learning. Recently, \citet{betser2026infonce} proved that the population-level InfoNCE objective asymptotically induces a Gaussian distribution when properly regularized. Within generative modeling, framing batch optimization as a particle-based objective has shown promise for maintaining global coverage \citep{deng_generative_2026}. While contrastive losses in continuous flows remain largely underexplored, \citet{lee2026aligning} use them to improve text-to-image alignment, \citet{kim2025diverse} apply them at inference time to enhance sample diversity, and \citet{stoica2025contrastive} separate conditional trajectories via contrastive velocity penalties. Our work is the first to leverage the asymptotic Gaussianity of contrastive learning to optimize the source distribution directly during flow matching training.

\section{Preliminaries}
\label{sec:preliminaries}

\subsection{Conditional Flow Matching and Stochastic Interpolants}

Conditional Flow Matching \citep{lipman2023flow, albergo2023stochastic, liu2023flow, tong_improving_2024} provides a simulation-free approach to train continuous normalizing flows by constructing a time-dependent probability path between a tractable source noise distribution $\mu_0 \in \mathcal{P}(\mathbb{R}^d)$ (typically $\mathcal{N}(0, \mathbf{I})$) and a complex target data distribution $\mu_1 \in \mathcal{P}(\mathbb{R}^d)$. Formally, let $\Pi(\mu_0, \mu_1) \subset \mathcal{P}(\mathbb{R}^d \times \mathbb{R}^d)$ denote the set of joint distributions that have $\mu_0$ and $\mu_1$ as their marginals (\replace{e.g.,}{i.e.,} the set of valid \textit{couplings}). For any \remove{such} valid coupling $\pi \in \Pi$, \add{let $(x_0,x_1)\sim \pi$ denote a paired sample, where $x_0\in\mathbb{R}^d$ is drawn from the source marginal $\mu_0$ and $x_1\in\mathbb{R}^d$ is drawn from the target marginal $\mu_1$.} \add{Then, with} an interpolant function $\phi_t(x_0, x_1) \in \mathbb{R}^d$ \add{$, t \in [0, 1]$}, subject to the boundary conditions $\phi_0(x_0, x_1) = x_0$ and $\phi_1(x_0, x_1) = x_1$, the continuous probability path via the pushforward operation is defined as $\rho_t = (\phi_t)_\# \pi$, $\rho_t  \in \mathcal{P}(\mathbb{R}^d)$. By construction, this ensures $\rho_0 = \mu_0$ and $\rho_1 = \mu_1$.

The dynamics of this interpolant are governed by a probability flow ODE, driven by a time-varying marginal vector field $v_t$ \citep{albergo2023stochastic}:
\begin{equation}
    dx_t = v_t(x_t) dt, \quad x_0 \sim \mu_0.
    \label{eq:prob_flow_ode}
\end{equation}
Directly learning $v_t$ is intractable as it requires knowledge of the marginal density $\rho_t$. Instead, Flow Matching relies on constructing tractable \textit{conditional} vector fields $u_t(x_t \mid x_0, x_1)$ that generate the conditional paths $\phi_t(x_0, x_1)$. The marginal vector field is then recovered via marginalization: $v_t(x) = \mathbb{E}_{\pi} [u_t(x_t \mid x_0, x_1) \mid x_t = x]$.

A widely used specific choice is the \textit{linear interpolant}, defined as $x_t := \phi_t(x_0, x_1) = (1 - t)x_0 + t x_1$ \citep{lipman2023flow}. For this linear path, the conditional velocity term simplifies to a constant vector: $x_1 - x_0$. To learn this target vector field, a neural network $v_\theta(x_t, t)$ is optimized to predict the flow by minimizing the conditional flow matching objective:
\begin{equation}
    \mathcal{L}_{\text{CFM}}(\theta) = \mathbb{E}_{\substack{t \sim \mathcal{U}[0,1] \\ (x_0, x_1) \sim \pi}} \left[ \| v_\theta(x_t, t) - (x_1 - x_0) \|_2^2 \right].
    \label{eq:prob_flow_loss}
\end{equation}

Once $v_\theta$ accurately models the true field $v_t$, we can map pure noise into the target data distribution by evolving \eqref{eq:prob_flow_ode}.

The choice of the coupling $\pi$ significantly dictates the optimization dynamics. An independent coupling, $\pi(x_0, x_1) = \mu_0(x_0)\mu_1(x_1)$, pairs samples entirely at random. This stochastic construction averages out opposing directions, inducing highly curved marginal sampling trajectories.
\paragraph{OT-CFM. } To mitigate this, \textit{Optimal Transport (OT)} couplings explicitly pair geometrically closer samples by minimizing a transport cost, typically the squared $L_2$ distance:
\begin{equation}
    \pi_{\text{OT}} = \arg \min_{\pi \in \Pi(\mu_0, \mu_1)} \int \|x_0 - x_1\|_2^2 \, d\pi(x_0, x_1).
    \label{eq:batch_ot}
\end{equation}
To avoid intractable dataset-wide computations, practical implementations utilize \textit{Minibatch OT} \citep{tong_improving_2024, pooladian_multisample_2023}. By matching samples discretely per batch, this strategy empirically straightens the marginal vector fields.

\subsection{Information Noise-Contrastive Estimation (InfoNCE)}
\paragraph{Empirical InfoNCE Loss.} We measure alignment using cosine similarity. For any vector $v \in \mathbb{R}^d$, let $\hat{v} = v / \|v\|_2$ denote its $L_2$-normalized counterpart, such that the cosine similarity between two vectors is given by their dot product $\langle \hat{z}, \hat{x} \rangle$. By choosing this metric, the InfoNCE loss \citep{oord2018representation, chen2020simple, kim2021selfreg} evaluates alignment strictly based on angular distance, effectively operating on the unit hypersphere $\mathcal{S}^{d-1}$.

Let $\pi$ denote the joint distribution of paired samples. Given a batch of $N$ pairs $\{(z_i, x_i)\}_{i=1}^N$ drawn i.i.d. from $\pi$, the empirical InfoNCE loss is formulated as
\begin{equation}
    \mathcal{L}_{\text{InfoNCE}, \tau} = - \frac{1}{N} \sum_{i=1}^N \log \frac{\exp(\langle \hat{z}_i, \hat{x}_i \rangle / \tau)}{\sum_{j=1}^N \exp(\langle \hat{z}_i, \hat{x}_j \rangle / \tau)}
    \label{eq:empirical_infonce}
\end{equation}
where a fixed temperature hyperparameter $\tau > 0$ regulates distribution sharpness.

Intuitively, the numerator maximizes the cosine similarity of the true positive pair $(z_i, x_i)$. Simultaneously, the denominator computes a partition function over all available candidates $\{x_j\}_{j=1}^N$ in the batch, where the $j \neq i$ instances serve as negative repulsors.

\paragraph{Population InfoNCE.} As the batch size $N \to \infty$, the empirical InfoNCE loss converges to a population-level functional at an $\mathcal{O}(N^{-1/2})$ rate \citep{wang2020understanding}. In the symmetric case, where positive pairs $(z, x) \sim \pi$ share identical marginals $\mu$, this limit elegantly decomposes the objective into two distinct geometric terms:
\begin{equation}
\begin{aligned}
    \lim_{N \to \infty} & \left( \mathcal{L}_{\text{InfoNCE}, \tau} - \log N \right) \\
    &= - \frac{1}{\tau} \mathbb{E}_{(z,x)\sim\pi} [\langle \hat{z}, \hat{x} \rangle] + \Phi_\tau(\mu)
\end{aligned}
\label{eq:population_infonce}
\end{equation}
where the second term is defined as:
\begin{equation}
    \Phi_\tau(\mu) := \mathbb{E}_{z\sim\mu} \left[ \log \mathbb{E}_{x\sim\mu} [\exp(\langle \hat{z}, \hat{x} \rangle / \tau)] \right]
    \label{eq:phi_term}
\end{equation}

The first term measures the alignment of positive pairs tightly coupled by $\pi$. The second term, $\Phi(\mu)$, acts as a uniformity potential that depends solely on the marginal distribution.

\section{Proposed Method}
\label{sec:method}

\subsection{Relaxed Prior Couplings}

Standard continuous normalizing flows and optimal transport enforce strict marginal constraints $\pi \in \Pi(\mu_0, \mu_1)$, where $\mu_0 = \mathcal{N}(0, \mathbf{I})$ is the prior and $\mu_1 \in \mathcal{P}(\mathbb{R}^d)$ is the data. This rigid boundary forces the transport map to resolve the entirety of the structural mismatch between an isotropic prior and clustered data, inevitably inducing tangled, high-curvature sampling trajectories.

To resolve this bottleneck, we extend the classical formulation into the regime of \textit{semi-unbalanced optimal transport} \citep{chizat2018unbalanced}: rather than enforcing a strict boundary at $t=0$, we allow the source distribution to be an adaptive measure $\tilde{\mu}_0 \in \mathcal{P}(\mathbb{R}^d)$. We define \textit{relaxed prior couplings} as joint measures $\pi \in \Pi(\cdot, \mu_1)$, where the exact second marginal matches the data, but the first marginal, denoted $\tilde{\mu}_0$, remains a free variable. We argue that relaxing these rigid prior constraints is not a compromise but a necessity for fast sampling, significantly simplifying the transport map.

To systematically constrain the structural optimization of this adaptive source, we formulate a divergence-regularized transport objective. Assuming absolute continuity of the source with respect to the prior ($\tilde{\mu}_0 \ll \mu_0$) to ensure finite divergence, we define:
\begin{equation}
\begin{aligned}
    \mathcal{W}_{c,\text{KL}}(\mu_0, \mu_1) = \inf_{\pi \in \Pi(\cdot, \mu_1)} \Big[ &-\frac{1}{\tau}\mathbb{E}_{(z,x)\sim\pi} [\langle \hat{z}, \hat{x} \rangle] \\
    &+ \beta \text{KL}(\tilde{\mu}_0 \parallel \mu_0) \Big]
\end{aligned}
    \label{eq:relaxed_coupling}
\end{equation}
where $\beta > 0$ controls the regularization strength.

This objective strategically shifts the modeling burden. The expected cosine similarity actively aligns the source marginal $\tilde{\mu}_0$ to structurally mirror $\mu_1$, minimizing the required transport effort. Concurrently, the KL divergence term penalizes deviations from the reference Gaussian $\mu_0$. While $\beta \to \infty$ recovers a hard-constrained OT problem with
spherical \textit{cosine} cost, a finite $\beta$ yields an optimal intermediate source that perfectly balances semantic alignment with prior tractability.

\subsection{Contrastive Noise Alignment}

We propose \textit{Contrastive Noise Alignment (CNA)} as a tractable, implicit \textit{empirical} realization of relaxed prior couplings. Standard flow matching assumes a fixed prior $z \sim \mathcal{N}(0, \mathbf{I})$, forcing the velocity field to resolve severe structural mismatches via highly curved transport trajectories. CNA flips this paradigm: instead of treating a sampled noise batch as static vectors, we model it as a particle system. By actively optimizing the noise batch's geometry to reflect the semantic topology of the target data, CNA effectively minimizes the relaxed objective (Equation \ref{eq:relaxed_coupling}).

We achieve this through an equilibrium of three geometric forces (Figure \ref{fig:cno_mechanics}): (1) a \textit{cross-modal alignment force} (transport cost) pulling noise toward assigned targets, (2) an \textit{angular entropy force} maximizing directional diversity through repulsion, and (3) a \textit{radial gravitational force} anchoring vectors to the origin. Together, the latter two explicitly tether the batch to the standard Gaussian prior, approximating the angular part of the KL penalty (App. \ref{app:kl_connection}).

\paragraph{Initial Target Assignment.} Directly optimizing the continuous joint measure $\pi$ (Equation \ref{eq:relaxed_coupling}) is computationally prohibitive. To make this tractable, CNA employs an empirical relaxation by fixing a discrete initial coupling $\pi_{\text{init}}$ (drawn randomly or via minibatch OT, Equation \ref{eq:batch_ot}) to pair each noise particle $z_i$ with a target $x_i$. Crucially, $\pi_{\text{init}}$ remains frozen during training. Rather than updating the coupling matrix, CNA optimizes the spatial coordinates of $z_i$. Under this strictly fixed assignment, the targets $x_i$ act as stationary anchors, pulling the free-floating noise particles into a topologically aligned configuration.

\paragraph{Cross-Modal Alignment.} To approximate the \textit{transport cost $c(z, x)$} and actively sculpt the noise distribution, we align normalized noise particles $\hat{z}_i$ with assigned targets $\hat{x}_i$ via an InfoNCE objective. Given a batch of size $N$, the alignment loss is formulated as
\begin{equation}
    \mathcal{L}_{\text{align}, \tau} = - \frac{1}{N} \sum_{i=1}^{N} \log \frac{\exp(\langle \hat{z}_i, \hat{x}_i \rangle / \tau)}{\sum_{j=1}^{N} \exp(\langle \hat{z}_i, \hat{x}_j \rangle / \tau)}
\end{equation} 
where $\tau > 0$ is a temperature scaling parameter. By restricting the gradient updates  to the noise particles, this optimization mechanism allows the noise vectors to self-organize into semantic clusters around the fixed data points, aligning the topology of the initial noise space with the underlying semantic structure of the data manifold. This is a fundamental advantage over standard \textit{Minibatch Optimal Transport} \citep{pooladian_multisample_2023, tong_improving_2024}, which operates on a fixed, randomly sampled set of noise vectors and is thus restricted to point-to-point assignments. 

The repulsive force in the denominator also serves as a key safeguard. By keeping the noise vectors repelled from the broader data manifold, it prevents them from collapsing directly into the target data. This leaves the local attraction just strong enough to capture the fine-grained details unique to the assigned target.

\paragraph{Noise-Noise Repulsion.} To approximate the KL penalty and prevent spatial collapse, we apply a \textit{noise-noise repulsion} force to the $L_2$-normalized vectors $\hat{z}_i \in \mathcal{S}^{d-1}$:
\begin{equation}
    \mathcal{L}_{\text{entropy}, \gamma} = \frac{1}{N} \sum_{i=1}^{N} \log \sum_{j \neq i} \exp\left(\frac{\langle \hat{z}_i, \hat{z}_j \rangle}{\gamma}\right).
    \label{eq:uniformity}
\end{equation}
By penalizing high cosine similarities, this objective treats vectors as mutually repelling charged particles, where $\gamma >0$ is a separate temperature parameter controlling the repulsive field.

Crucially, this is not merely a dispersion heuristic. As we formally demonstrate in Appendix \ref{app:angular_entropy}, minimizing this term is equivalent to maximizing the empirical angular entropy $\hat{H}(\hat{Z})$, directly promoting the uniform distribution on the hypersphere in the asymptotic limit \citep{wang2020understanding}.

\paragraph{Gaussian Norm Regularizer.} Because $\mathcal{L}_{\text{align}}$ and $\mathcal{L}_{\text{entropy}}$ are scale-agnostic, optimizing them via discrete gradient steps provably (for GD) inflates vector magnitudes (see Appendix \ref{app:norm_growth}). To counteract this centrifugal expansion and to preserve the radial constraint of the prior, we apply a stabilizing $L_2$ penalty:
\begin{equation}
    \mathcal{L}_{\text{norm}} = \frac{1}{N} \sum_{i=1}^{N} \|z_i\|_2^2.
    \label{eq:norm}
\end{equation}
\paragraph{Final Objective.} The final objective of the CNA combines structural alignment ($\mathcal{L}_{\text{align}}$) with decoupled angular and radial divergence penalties ($\mathcal{L}_{\text{entropy}}$ and $\mathcal{L}_{\text{norm}}$):
\begin{equation}
    \mathcal{L}_{\text{total}} = \mathcal{L}_{\text{align}, \tau} + \beta \mathcal{L}_{\text{entropy}, \gamma} + \lambda \mathcal{L}_{\text{norm}}.
    \label{eq:total_loss}
\end{equation}
Iteratively minimizing this objective dynamically aligns the noise and target batches, yielding an optimized set of source particles $\{\tilde{z}_i\}_{i=1}^N$. These topologically aligned pairs $(\tilde{z}_i, x_i)$ replace the standard random samples $(x_0, x_1)$ in the Conditional Flow Matching objective (Equation \ref{eq:prob_flow_loss}). Importantly, our method only alters the training procedure, allowing inference to proceed exactly as usual. Implementation details are provided in Appendix \ref{app:implementation_details}.

\begin{figure*}[t]
    \centering
    \begin{subfigure}[b]{0.6\textwidth}
        \centering
        \small
        \scalebox{0.95}{ 
            \setlength{\tabcolsep}{2.0pt} 
            \begin{tabular}{@{} >{\columncolor{white}[0pt][\tabcolsep]}l *{7}{>{\centering\arraybackslash}p{0.9cm}} >{\columncolor{white}[\tabcolsep][0pt]\centering\arraybackslash}p{0.9cm} @{}}
            \toprule
            \multirow{2}{*}{\textbf{Method}} & \multicolumn{8}{c}{\textbf{FID $\downarrow$ (Euler NFEs)}} \\
            \cmidrule(lr){2-9} 
             & \textbf{1} & \textbf{2} & \textbf{4} & \textbf{8} & \textbf{16} & \textbf{32} & \textbf{64} & \textbf{128} \\
            \midrule
            I-CFM & 346.76 & 175.91 & 54.39 & 18.21 & 9.59 & 6.49 & 5.05 & 4.37 \\
            OT-CFM \citep{tong_improving_2024} & 229.98 & 91.75  & 30.00 & 14.20 & 9.08 & 6.58 & 5.12 & 4.36 \\
            \midrule
            CNA (ours, $\beta = 2)$ & 220.87 & 87.39 & 28.87 & 14.00 & 8.90 & 6.44 & 5.09 & 4.26 \\
            \midrule
            CNA + OT (ours, $\beta = 2$) & \textbf{148.54} & \textbf{56.40} & \textbf{22.64} & \textbf{12.28} & \textbf{8.23} & 6.36 & 5.51 & 5.25 \\
            \rowcolor{gray!10}
            CNA + OT (ours, $\beta = 5$) & 179.84 & 69.59 & 25.24 & 12.93 & 8.27 & \textbf{5.88} & \textbf{4.62} & \textbf{4.06} \\
            \bottomrule
            \end{tabular}
        }
        \caption{Unconditional generation results on CIFAR-10}
        \label{tab:results_tffid_simplified}
    \end{subfigure}
    \hfill 
    \begin{subfigure}[b]{0.37\textwidth}
        \centering
        \begin{tikzpicture}
            \begin{axis}[
                width=\linewidth, 
                height=4.3cm, 
                xmode=log, log basis x=2, ymode=log,
                xtick={1,2,4,8,16,32, 64,128}, xticklabels={1,2,4,8,16,32, 64,128},
                ytick={5, 10, 20, 50, 100, 200, 400}, yticklabels={5, 10, 20, 50, 100, 200, 400},
                ylabel={{FID} $\downarrow$},
                ylabel style={font=\footnotesize, yshift=-10pt},
                label style={font=\footnotesize}, tick label style={font=\scriptsize},
                legend pos=north east, 
                legend style={font=\tiny, fill=white, fill opacity=0.9, inner sep=0.1pt, nodes={yshift=0.5pt}},
                grid=both, grid style={dashed, gray!30},
                every axis plot/.append style={thick, mark size=1.25pt},
                ymin=5, ymax=390
            ]
            \addplot[gray, dashed, mark=*] coordinates {(1, 346.76)(2,175.91) (4,54.39) (8,18.21) (16,9.59) (32,6.49)};
            \addlegendentry{I-CFM}
            \addplot[blue, dashed, mark=triangle*] coordinates {(1, 229.98)(2,91.75) (4,30.00) (8,14.20) (16,9.08) (32,6.58)};
            \addlegendentry{OT-CFM}
            \addplot[green!50!black, mark=square*] coordinates {(1, 148.54)(2,56.4) (4,22.64) (8,12.28) (16,8.23)(32,6.36)};
            \addlegendentry{CNA + OT (ours, $\beta = 2$)}
            \end{axis}
        \end{tikzpicture}
        \caption{FID across Euler NFEs.}
        \label{fig:results_plot_log}
    \end{subfigure}
    
    \caption{\textbf{Quantitative results on CIFAR-10.} (a) Table showing FID scores; (b) Plot visualizing performance versus sampling budget.}
    \label{fig:combined_cifar_results}
\end{figure*}
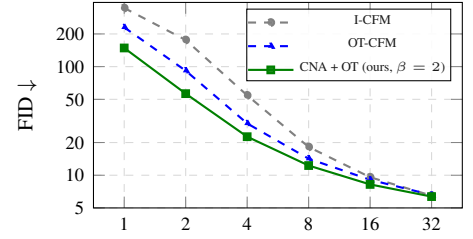

\subsection{Theoretical Intuitions and Induced Gaussianity}
\label{sec:theoretical_intuitions}

\paragraph{The Core Intuition:} Contrastive Noise Alignment (CNA) is grounded in the \textit{Maxwell-Poincaré spherical central limit theorem} \citep{maxwell1860ii, poincare1912calcul, diaconis1984asymptotics}, which states that fixed-dimensional projections of a uniform distribution on a scaled hypersphere converge to a Gaussian as the dimension grows:
\begin{lemma}[Maxwell-Poincaré; \cite{diaconis1984asymptotics}]
\label{lem:maxwell_poincare}
Let $\sigma$ denote the uniform distribution on $\mathcal{S}^{d-1}$. As $d \to \infty$, for every fixed $k \ge 1$, the $k$-dimensional marginal of $u \sim \sigma$ satisfies
\begin{equation}
    \sqrt{d} u_k \Rightarrow \mathcal{N}(0, \mathbf{I}_k),
\end{equation}
where $u_k$ denotes the projection of $u$ onto a fixed $k$-dimensional subspace. The total variation distance between $\sqrt{d} u_k$ and $\mathcal{N}(0, \mathbf{I}_k)$ converges at a rate of $\mathcal{O}(d^{-1})$ \citep{diaconis1987dozen}.
\end{lemma}
By enforcing angular uniformity via $\mathcal{L}_{\text{entropy}}$ and preventing magnitudes from exploding via $\mathcal{L}_{\text{norm}}$, our objective provides the necessary structural forces to drive the noise particles toward a standard normal distribution in high dimensions.

To formalize this mechanism, we map our setup to the population InfoNCE objective of \citet{wang2020understanding} (i.e., Equation \ref{eq:population_infonce}) and a framework similar to that of \citet{betser2026infonce}. Let $\mu_0 = \mathcal{N}(0, \mathbf{I})$ be the initial noise distribution, and $\tilde{\mu}_0 = g_\# \mu_0$ be the pushforward distribution generated by our optimization loop $g$. \citet{betser2026infonce} prove that optimizing a bounded contrastive objective forces $\tilde{\mu}_0$ to minimize a regularized function, asymptotically inducing a Gaussian distribution.

\paragraph{Connection to Population InfoNCE.} First, we observe that our empirical entropy loss $\mathcal{L}_{\text{entropy}, \gamma}$ asymptotically recovers the unimodal repulsive component of the population InfoNCE objective (Equation \ref{eq:population_infonce}). Leveraging this equivalence, we can decouple the global objective into cross-modal and unimodal interactions, leading directly to the decomposition formalized in Theorem \ref{thm:asymptotic_decomp}:

\begin{theorem}[Asymptotic Decomposition]
\label{thm:asymptotic_decomp}
Let $\pi$ denote the joint distribution of matched positive pairs $(z, x)$, with respective marginals $\tilde{\mu}_0$ for the noise distribution and $\mu_1$ for the target data. Let $\hat{z}$ and $\hat{x}$ denote their $L_2$-normalized representations. For fixed parameters $\beta, \lambda > 0$, as the batch size $N \to \infty$, our combined objective converges a.s. to:
\begin{equation}
\begin{aligned}
    &\lim_{N \to \infty} \Big( \mathcal{L}_{\text{align}, \tau} + \beta \mathcal{L}_{\text{entropy}, \gamma} \\
    &\qquad\quad + \lambda \mathcal{L}_{\text{norm}} - (1 + \beta)\log N \Big) = \\
    &-\frac{1}{\tau} \mathbb{E}_{(z,x)\sim\pi} [\langle \hat{z}, \hat{x} \rangle] + \Phi_\tau(\tilde{\mu}_0, \mu_1) \\
    &+ \beta \Phi_\gamma(\tilde{\mu}_0) + \lambda\mathbb{E}_{z\sim\tilde{\mu}_0}\left[\|z\|^2\right]
\end{aligned}
\label{eq:limit_expansion_compact}
\end{equation}
where the unimodal repulsion $\Phi_\gamma(\tilde{\mu}_0)$ and the cross-modal repulsion $\Phi_\tau(\tilde{\mu}_0, \mu_1)$ are defined as:
\begin{align*}
    \Phi_\gamma(\tilde{\mu}_0) &= \mathbb{E}_{z_i\sim\tilde{\mu}_0} \left[ \log \mathbb{E}_{z_j\sim\tilde{\mu}_0} \left[ \exp(\langle \hat{z}_i, \hat{z}_j \rangle / \gamma) \right] \right] \\
    \Phi_\tau(\tilde{\mu}_0, \mu_1) &= \mathbb{E}_{z\sim\tilde{\mu}_0} \left[ \log \mathbb{E}_{x\sim\mu_1} \left[ \exp(\langle \hat{z}, \hat{x} \rangle / \tau) \right] \right]
\end{align*}
\end{theorem}

\begin{proof} The full derivation is provided in Appendix \ref{sec:proof_theorem_1}.
\end{proof}

\paragraph{Induced Gaussianity. }
Let $\sigma$ denote the uniform distribution on $\mathcal{S}^{d-1}$. Crucially, \citet[Appendix A]{wang2020understanding} prove that the uniformity potential $\Phi_\gamma(\tilde{\mu}_0)$ is uniquely minimized at the uniform distribution on the sphere (i.e., $\tilde{\mu}_0 = \sigma$).

\begin{corollary}[Induced Gaussianity as $\beta \to \infty$]
\label{cor:induced_gaussianity}
Under the conditions of Theorem \ref{thm:asymptotic_decomp} with initialization $\mu_0 = \mathcal{N}(0, \mathbf{I})$ and a scale-agnostic objective ($\lambda = 0$), let $\tilde{\mu}_0$ be the asymptotic global minimizer as $\beta \to \infty$. Then, for any fixed $k \ge 1$, the $k$-dimensional projections of the unnormalized representations $z \sim \tilde{\mu}_0$ converge in distribution to $\mathcal{N}(0, \mathbf{I}_k)$ as $d \to \infty$.
\end{corollary}
The result follows from uniquely minimizing $\Phi_\gamma$ at $\sigma$ via Lemma \ref{lem:maxwell_poincare}, paired with the initial gaussian thin-shell concentration preserved under scale-invariance. See Appendix \ref{sec:proof_corollary_1} for details. \textit{(Note: As discussed in Section \ref{sec:method} and Appendix \ref{app:norm_growth}, empirical discrete optimization necessitates a norm regularizer to counteract the outward norm drift.)}

\paragraph{Practical Considerations.} While practical settings operate with finite dimensions and batch sizes, these theoretical limits still serve as valuable motivating approximations for training. Formally, the deviation of the expected contrastive loss from its population limit decays as $\mathcal{O}(M^{-1/2})$ in the number of negative samples $M$\citep{wang2020understanding}. Meanwhile, the high-dimensional projection error scales as $\mathcal{O}(d^{-1})$ \cite{diaconis1987dozen}. When $\beta$ is sufficiently large, our experiments confirm that the optimized source distribution remains sufficiently close to a standard Gaussian prior.

\section{Experiments}
We evaluate CNA for unconditional image generation on CIFAR-10 and ImageNet32. We also perform an extensive ablation study to assess the contribution of each model component.

\subsection{CIFAR-10 (Unconditional)}
\label{subsec:cifar10}

\paragraph{Experimental Setup. }To validate our approach, we first test CNA on the standard CIFAR-10 benchmark, which contains $50,000$ training images of resolution $32\times32$. I-CFM, OT-CFM, and CNA training differ exclusively in the way noise-data pairs are sampled, all other aspects of FM training stay identical, using the setup from \citet{tong_improving_2024}. All reported baseline numbers are reproduced. During inference, we sample purely from the exact Gaussian prior, $x_0 \sim \mathcal{N}(0, \mathbf{I})$. The full training details are provided in Appendix \ref{app:implementation_details}.
\paragraph{Evaluation Metrics.} We measure sample quality via Fr{\'e}chet Inception Distance (FID) \citep{heusel2017gans} and path straightness via flow curvature \citep{lee2023minimizing}. To test few-step generation, we report FID using a deterministic Euler solver across 1 to 128 number of function evaluations (NFEs).
\paragraph{Results.} We present our quantitative results in Figure \ref{fig:combined_cifar_results} and Table \ref{tab:flow_curvature}. Optimizing the noise prior via CNA consistently outperforms both I-CFM and OT-CFM across all generation budgets. At $2$ sampling steps, CNA reduces FID from $175.91$ to $87.39$, a ${\sim}50\%$ reduction. This lead remains robust across few-step Euler budgets: CNA reaches 4-step and 8-step FIDs of 28.87 and 14.00, respectively ($\beta = 2$). Concurrently, CNA reduces flow curvature (Euler-128 $\kappa$ drops ${\sim}32\%$ from $0.0479$ to $0.0324$), confirming that our aligned prior inherently straightens generation trajectories.
\begin{table}[H]
    \centering
    \small
    \setlength{\tabcolsep}{3pt} 
    \scalebox{0.95}{
    \begin{tabular}{| l | c | c | c | c |} 
        \hline
        & I-CFM & OT-CFM & CNA ($\beta = 2$) & CNA+OT ($\beta = 2$) \\
        \hline
        Euler-128 & 0.0479 & 0.0326 & 0.0324 & \textbf{0.0251} \\
        Dopri5* & 0.0531 & 0.0290 & 0.0280 & \textbf{0.0185} \\
        \hline
    \end{tabular}
    }
    \caption{Curvature $\kappa$ ($\downarrow$) for \textbf{CIFAR-10}. \textit{Note*: Dopri5 $\kappa$ deviates from the uniform-time definition.}}
    \label{tab:flow_curvature}
\end{table}
The Initialization of CNA with OT couplings (CNA+OT) provides further gains across all evaluation budgets. Using $\beta = 2$, this combination reduces the $2$-step FID to $56.40$ (an ${\sim}38\%$ improvement over OT-CFM) and achieves the lowest curvature ($\kappa=0.0251$ at Euler-128). At $\beta = 5$, CNA+OT also improves the high-NFE limit, reaching an FID of $4.06$ at Euler-128. This suggests that the optimized prior successfully retains the essential characteristics of a standard Gaussian prior if $\beta$ is chosen high enough, while also providing a beneficial warm-start. By starting from an already efficient transport plan, the contrastive forces avoid large-scale structural reshuffling and can focus purely on fine-grained alignment. Qualitative samples reflect these metrics, producing much sharper images at low NFEs (Figure~\ref{fig:nfe_visual_comparison}). Since CNA operates entirely during training, these benefits come at zero additional inference cost.
\subsection{ImageNet32 (Unconditional)}
To test the scalability of our approach on a significantly more complex and diverse data distribution, we extend our unconditional evaluation to the ImageNet32 dataset, which comprises approximately $1.28$ million training images downsampled to a $32 \times 32$ resolution. As with the CIFAR-10 experiments, we maintain a fixed network architecture, regression objective, and training budget across all baselines, strictly isolating the impact of the noise coupling.
\begin{table}[H]
    \centering
    \small
    \caption{Unconditional generation results on \textbf{ImageNet32}.}
    \label{tab:results_imagenet32}
    \vskip 0.05in
    \setlength{\tabcolsep}{3.5pt} 
    \scalebox{0.95}{
    \begin{tabular}{@{} >{\columncolor{white}[0pt][\tabcolsep]}l *{4}{c} >{\columncolor{white}[\tabcolsep][0pt]}c @{}}
    \toprule
    \multirow{2}{*}{\textbf{Method}} & \multicolumn{5}{c}{\textbf{FID $\downarrow$ (Euler NFEs)}} \\
    \cmidrule(lr){2-6} 
     & \textbf{1} & \textbf{2} & \textbf{4} & \textbf{8} & \textbf{16} \\
    \midrule
    I-CFM & 419.92 & 205.99 & 69.51 & 25.05 & 12.14 \\
    OT-CFM  \citep{tong_improving_2024}& 252.88 & 102.63 & 36.95 & 15.89 & 9.41  \\
    \midrule
    CNA (ours, $\beta = 2$)                       & 250.87 & 105.67  & 39.93 & 17.85 & 10.89   \\
    \midrule
    CNA + OT (ours, $\beta = 5)$                  & \textbf{166.07} & \textbf{71.78} & \textbf{26.86} & 15.03 & 12.82 \\
    CNA + OT (ours, $\beta = 30)$                  & 230.85 & 93.33 & 34.02 & \textbf{14.78} & \textbf{8.96} \\
    \bottomrule
    \end{tabular}
    }
\end{table}
\paragraph{Results.} Table~\ref{tab:results_imagenet32} summarizes our quantitative evaluation on ImageNet32. CNA combined with OT initialization demonstrates substantial improvements in the highly constrained few-step regime. At $\beta=5$, our method drops the 1-step Euler FID from $252.88$ (OT-CFM baseline) to $166.07$—a massive ${\sim}34\%$ reduction. This robust performance extends through the 2-step and 4-step budgets, reducing FIDs to $71.78$ (a ${\sim}30\%$ improvement) and $26.86$, respectively. For higher step counts, increasing the entropy regularization ($\beta=30$) secures the high-NFE limit, pushing the 16-step FID down to $8.96$. As observed in previous experiments, these generation gains are structurally supported by a concurrent reduction in flow curvature (Table~\ref{tab:flow_curvature_imagenet}).
\begin{table}[H]
    \centering
    \small
    \setlength{\tabcolsep}{3pt} 
    \scalebox{0.95}{
    \begin{tabular}{| l | c | c | c | c |} 
        \hline
        & I-CFM & OT-CFM & CNA ($\beta = 2$) & CNA+OT ($\beta = 10$) \\
        \hline
        Euler-128 & 0.0600 & 0.0416 & 0.0412 & \textbf{0.0381} \\
        Dopri5* & 0.0629 & 0.0344 & 0.0326 & \textbf{0.0283} \\
        \hline
    \end{tabular}
    }
    \caption{Curvature $\kappa$ ($\downarrow$) for \textbf{ImageNet32.}}
    \label{tab:flow_curvature_imagenet}
\end{table}
\subsection{Analysis}
\paragraph{Ablations.}We perform a step-by-step ablation on CIFAR-10 to evaluate the contribution of each mechanism (Table~\ref{tab:ablation_table}). Applying only the contrastive alignment term pulls the noise distribution toward the data topology. This improves few-step generation, lowering the Euler-4 FID from $30.00$ to $22.60$ (at $\tau=0.01$). However, this unconstrained attraction induces spatial collapse (e.g., prior hole problem \citet{hao2023coupled}), degrading adaptive solver quality (Dopri5 FID rises to $10.52$). 

\begin{table}[t]
    \centering
    \scalebox{0.88}{
    \begin{tabular}{@{}lccc@{}}
        \toprule
        \multirow{2}{*}{\textbf{Configuration}} & \multicolumn{3}{c}{\textbf{FID $\downarrow$}} \\
        \cmidrule(l){2-4}
        & Euler-4 & Euler-10 & Dopri5 \\
        \midrule
        
        OT-CFM \textit{(baseline)} \citep{tong_improving_2024} & 30.00 & 12.13 & 3.66 \\
        
        \midrule
        \quad +Alignment ($\tau = 0.1$) & 25.32 & 15.81 & 15.84 \\
        \quad +Alignment ($\tau = 0.01$)\textsuperscript{def} & \textbf{22.60} & 12.89 & 10.52 \\
        \addlinespace
        \quad +Align + Ent ($\tau, \gamma = 0.01$)\textsuperscript{def} & 26.14 & 11.64 & 3.60 \\
        
        \midrule
        CNA + OT (Align + Ent + $\mathcal{L}_{\text{norm}}$) & 25.24 & \textbf{11.02} & \textbf{3.56} \\
        
        \bottomrule
    \end{tabular}}
    \caption{Step-by-step performance impact of Alignment, Entropy, and Norm regularization on \textbf{CIFAR-10}, starting from the OT-CFM baseline. Using $\beta$ = 5. Default parameters are marked with \textsuperscript{def}.}
    \label{tab:ablation_table}
\end{table}
To address this, Entropy Repulsion acts as a uniformity potential against prior collapse, while Norm Regularization ($\mathcal{L}_{\text{norm}}$) prevents the optimized vectors from escaping the $\mathcal{N}(0, \mathbf{I})$ shell. As the final table row indicates, combining all three components yields an optimal balance: The complete CNA formulation accelerates few-step generation (reducing Euler-4 and Euler-10 FIDs to $25.24$ and $11.02$) while restoring asymptotic quality (Dopri5 FID of $3.56$).
\begin{figure*}[t]
    \centering
    \small

    \begin{subfigure}[t]{0.48\textwidth}
        \centering
        \small
        \setlength{\tabcolsep}{1.5pt}
        \scalebox{0.8}{
        \begin{tabular}{ccccc}
            & \small \textbf{1 NFE} & \small \textbf{2 NFE} & \small \textbf{4 NFE} & \small \textbf{8 NFE} \\

            \rotatebox{90}{\makebox[1.1cm][c]{\small \textbf{I-CFM}}} &
            \includegraphics[width=0.18\linewidth]{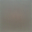} &
            \includegraphics[width=0.18\linewidth]{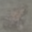} &
            \includegraphics[width=0.18\linewidth]{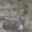} &
            \includegraphics[width=0.18\linewidth]{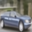} \\

            \noalign{\vskip 2pt}

            \rotatebox{90}{\makebox[1.1cm][c]{\small \textbf{OT-CFM}}} &
            \includegraphics[width=0.18\linewidth]{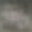} &
            \includegraphics[width=0.18\linewidth]{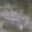} &
            \includegraphics[width=0.18\linewidth]{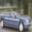} &
            \includegraphics[width=0.18\linewidth]{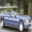} \\

            \noalign{\vskip 2pt}

            \rotatebox{90}{\makebox[1.1cm][c]{\small \textbf{Ours}}} &
            \includegraphics[width=0.18\linewidth]{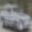} &
            \includegraphics[width=0.18\linewidth]{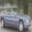} &
            \includegraphics[width=0.18\linewidth]{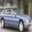} &
            \includegraphics[width=0.18\linewidth]{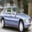} \\
        \end{tabular}
        }
        \caption{CIFAR-10}
    \end{subfigure}
    \hfill
    \begin{subfigure}[t]{0.48\textwidth}
        \centering
        \setlength{\tabcolsep}{1.5pt}
        
        \newcommand{\imgidx}{00512}
        
        \scalebox{0.8}{
        \begin{tabular}{ccccc}
            & \small \textbf{1 NFE} & \small \textbf{2 NFE} & \small \textbf{4 NFE} & \small \textbf{8 NFE} \\

            \rotatebox{90}{\makebox[1.1cm][c]{\small \textbf{I-CFM}}} &
            \includegraphics[width=0.18\linewidth]{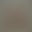} &
            \includegraphics[width=0.18\linewidth]{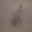} &
            \includegraphics[width=0.18\linewidth]{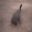} &
            \includegraphics[width=0.18\linewidth]{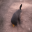} \\

            \noalign{\vskip 2pt}

            \rotatebox{90}{\makebox[1.1cm][c]{\small \textbf{OT-CFM}}} &
            \includegraphics[width=0.18\linewidth]{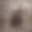} &
            \includegraphics[width=0.18\linewidth]{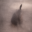} &
            \includegraphics[width=0.18\linewidth]{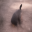} &
            \includegraphics[width=0.18\linewidth]{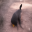} \\

            \noalign{\vskip 2pt}

            \rotatebox{90}{\makebox[1.1cm][c]{\small \textbf{Ours}}} &
            \includegraphics[width=0.18\linewidth]{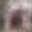} &
            \includegraphics[width=0.18\linewidth]{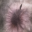} &
            \includegraphics[width=0.18\linewidth]{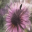} &
            \includegraphics[width=0.18\linewidth]{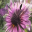} \\
        \end{tabular}
        }
        \caption{ImageNet32}
    \end{subfigure}
    \caption{\textbf{Qualitative comparison} across different step counts. Compared to the baselines, our method (CNA + OT) maintains structural coherence even at very low NFEs (1 and 2 steps).}
    \label{fig:nfe_visual_comparison}
\end{figure*}
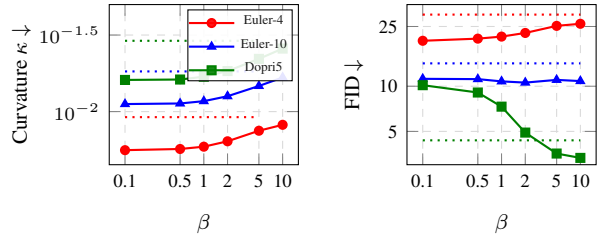
\begin{figure}[H]
    \centering
    
    \begin{subfigure}[b]{0.49\textwidth} 
        \centering
        \begin{tikzpicture}
            \begin{axis}[
                width=\linewidth, 
                height=3.7cm, 
                xmode=log, 
                xtick={0.1, 0.5, 1, 2, 5, 10}, xticklabels={0.1, 0.5, 1, 2, 5, 10},
                xlabel={$\beta$}, 
                ylabel={{Curvature} $\kappa$ $\downarrow$},
                ylabel style={font=\footnotesize, yshift=-2pt},
                ymode=log, log basis y=10,
                label style={font=\footnotesize}, tick label style={font=\scriptsize},
                legend pos=north east, 
                legend style={font=\tiny, fill=white, fill opacity=0.9, inner sep=0.0pt, nodes={yshift=0.5pt}},
                grid=both, grid style={dashed, gray!30},
                every axis plot/.append style={thick},
                ymin=0.0045, ymax=0.05
            ]
            \addplot[red, dotted, thick, mark=none, forget plot] coordinates {(0.1,0.0092) (5,0.0092)};
            \addplot[blue, dotted, thick, mark=none, forget plot] coordinates {(0.1,0.0183) (5,0.0183)};
            \addplot[green!50!black, dotted, thick, mark=none, forget plot] coordinates {(0.1,0.0290) (5,0.0290)};
            \addlegendentry{Euler-4}
            \addlegendentry{Euler-10}
            \addlegendentry{Dopri5}
            \addplot[red, mark=*, mark size=1.5pt] coordinates {(0.1,0.0056)(0.5,0.0057)(1, 0.0059)(2,0.0064) (5,0.0075)(10,0.0082)};
            \addplot[blue, mark=triangle*, mark size=1.5pt] coordinates {(0.1,0.0112)(0.5,0.0113)(1,0.0117)(2,0.0126) (5,0.0147)(10,0.0167)};
            \addplot[green!50!black, mark=square*, mark size=1.5pt] coordinates {(0.1,0.0161)(0.5,0.0162)(1.0,0.0169)(2,0.0184)(5,0.0221)(10,0.0258)};
            \addlegendimage{no markers, black, dotted, thick}
            \end{axis}
        \end{tikzpicture}
        \label{fig:ablation_curvature}
    \end{subfigure}
    \hfill
    \begin{subfigure}[b]{0.49\textwidth} 
        \centering
        \begin{tikzpicture}
            \begin{axis}[
                width=\linewidth, 
                height=3.7cm, 
                xmode=log, 
                xtick={0.1, 0.5, 1, 2, 5, 10}, xticklabels={0.1, 0.5, 1, 2, 5, 10},
                ymode=log, log basis y=10, 
                ylabel={{FID} $\downarrow$},
                xlabel={$\beta$}, 
                ytick={5, 10, 25, 50, 100},
                yticklabels={5, 10, 25, 50, 100},
                ylabel style={font=\footnotesize, yshift=-15pt, xshift=0pt},
                label style={font=\footnotesize}, tick label style={font=\scriptsize},
                legend pos=north east, 
                legend style={font=\tiny, fill=white, fill opacity=0.9, inner sep=1.0pt, nodes={yshift=0.5pt}},
                grid=major, grid style={dashed, gray!30},
                every axis plot/.append style={thick},
                ymin=3, ymax=35
            ]
            \addplot[red, dotted, thick, mark=none, forget plot] coordinates {(0.1,30.00) (10,30.00)};
            \addplot[blue, dotted, thick, mark=none, forget plot] coordinates {(0.1,14.20) (10, 14.20)};
            \addplot[green!50!black, dotted, thick, mark=none, forget plot] coordinates {(0.1,4.36) (10,4.36)};
            \addplot[red, mark=*, mark size=1.5pt] coordinates {(0.1,20.08)(0.5,20.75)(1, 21.41)(2,22.64) (5,25.24)(10,26.06)};
            \addplot[blue, mark=triangle*, mark size=1.5pt] coordinates {(0.1,11.20)(0.5,11.15)(1,10.77)(2,10.57) (5,11.02)(10,10.82)};
            \addplot[green!50!black, mark=square*, mark size=1.5pt] coordinates {(0.1,10.16)(0.5,9.08)(1.0,7.31)(2,4.91)(5,3.56)(10,3.33)};
            \addlegendimage{no markers, black, dotted, thick}
            \end{axis}
        \end{tikzpicture}
        \label{fig:ablation_fid}
    \end{subfigure}
    \vspace{-0.4cm}
    \caption{Trajectory curvature and FID with regularization parameter $\beta$ on \textbf{CIFAR-10}. \textit{Dotted lines denote the OT-CFM baseline.}}
    \label{fig:ablation_plots}
\end{figure}
\paragraph{Noise-Target Alignment.}To verify structural alignment, we measure the average batch cosine similarity for the matched noise-image pairs (Table~\ref{tab:batch_correlation}). While I-CFM yields zero correlation and OT-CFM provides an initial alignment ($0.0379$), our full CNA+OT ($\beta = 5$) method drives this further to $0.0597$, an over $1.5\times$ improvement over OT-CFM. A similar trend holds on ImageNet32, where CNA+OT ($\beta = 10$) increases alignment from $0.0491$ to $0.0707$. This confirms that CNA+OT systematically reduces \textit{directional displacement} during transport.
\begin{table}[H]
    \centering
    \small
    \setlength{\tabcolsep}{3pt} 
    \scalebox{0.95}{
    \begin{tabular}{| l | c | c | c | c | c |}
        \hline
        & I-CFM & OT-CFM & CNA ($\beta = 2$) & CNA+OT\\
        \hline
        CIFAR-10   & 0.000 & 0.0379 & 0.0453 & \textbf{0.0597} \\
        ImageNet32 & 0.000 & 0.0491 & 0.0561 & \textbf{0.0707} \\
        \hline
    \end{tabular}
    }
    \caption{\textbf{Batch Correlation.} Average cosine similarity $\theta$ between coupled pairs.}
    \label{tab:batch_correlation}
\end{table}
\paragraph{Effect of Regularization Weight $\beta$.}Figure~\ref{fig:ablation_plots} illustrates how the weight $\beta$ balances trajectory straightness and prior fidelity. We train different models with various values of $\beta$ on CIFAR-10. Keeping $\beta$ low allows the noise particles to align closely with the data. This creates exceptionally straight paths, resulting in peak few-step performance. The downside is that the prior drifts from a true Gaussian, which hurts high-step quality on solvers like Dopri5. Conversely, higher $\beta$ values reduce path straightness but benefit Dopri5 performance, even outperforming the baselines.
\paragraph{Post-Hoc Path Regulation.}Interestingly, the flow matching network can be explicitly conditioned on the regularization parameter $\beta$ during training. This turns $\beta$ into a \textit{post-hoc inference parameter}, allowing users to tune the trade-off between path straightness and prior adherence at generation time without retraining the model (for more details, see Table \ref{tab:beta_results} and App. \ref{sec:beta_conditioning_appendix}).
\begin{table}[H]
\centering
\small
\caption{Inference profile of the $\beta$-conditioned model on \textbf{CIFAR-10}. Lower $\beta$ values optimize for few-step generation.}
\label{tab:beta_results}
\scalebox{0.95}{
\begin{tabular}{lcccc}
    \toprule
    \multirow{2}{*}{\textbf{Inference $\beta$}} & \multicolumn{4}{c}{FID $\downarrow$} \\
    \cmidrule(lr){2-5}
    & Euler-2 & Euler-4 & Euler-10 & Dopri5 \\
    \midrule
    OT-CFM & 91.75 & 30.00 & 12.13 & 3.66 \\
    \midrule
    $\beta = 1.5$ & \textbf{55.04} & \textbf{22.19} & \textbf{10.48} & 6.13 \\
    $\beta = 3.0$ & 62.11 & 23.83 & 10.54 & 4.25 \\
    \rowcolor{gray!10}
    $\beta = 6.0$ & 70.75 & 25.88 & 11.07 & \textbf{3.49} \\
    \bottomrule
\end{tabular}}
\end{table}
\paragraph{Computational Overhead.}While an inner optimization loop inevitably adds cost, the footprint of CNA remains manageable. Achieving effective alignment in just $T=4$ steps, it introduces a modest ${\sim}2\%$ overhead on both CIFAR-10 and ImageNet32 (see Appendix \ref{sec:comp_efficiency}). CNA scales quadratically ($\mathcal{O}(N^2)$) with batch size $N$, keeping it cheaper than the cubic ($\mathcal{O}(N^3)$) exact \textit{Minibatch OT} solvers \citep{pooladian_multisample_2023,tong_improving_2024}. Consequently, in the combined CNA + OT variant, the OT initialization remains the primary runtime bottleneck. Additionally, by operating entirely online per minibatch, CNA bypasses offline pre-computation over
the full dataset, which can reach hours at scale \citep{mousavi-hosseini_flow_2026}.

\section{Conclusion}
\label{sec:conclusion}

In this work, we introduced \textit{Contrastive Noise Alignment (CNA)}, a novel framework that optimizes the initial noise prior in flow matching models. CNA aligns source noise representations directly with the target data while approximately maintaining a global Gaussian distribution. In practice, this alignment significantly boosts generation quality when the sampling budget is limited. By offloading the task of discovering transport structure from the neural network to the prior itself, CNA provides a more direct, principled path toward learning geometrically aware forward processes.

\paragraph{Limitations \& Ongoing Work.} A key limitation of our method is the tension between aggressively aligning the noise prior and preserving its Gaussian properties. We are currently exploring new regularization methods to better constrain the adapted prior to the Gaussian manifold. Furthermore, while CNA optimizes the noise distribution, it is still initialized from random noise, meaning the starting point can be sub-optimal. A promising future direction is \textit{deterministic Gaussianization} \citep{laparra2011iterative}: constructing the prior directly from the data to bypass random sampling entirely.

Beyond these theoretical questions, we are currently extending CNA to conditional and latent generation, as well as testing our contrastive loss within semantic feature spaces like DINOv2 \citep{oquab2023dinov2}. Finally, because our approach is entirely orthogonal to the regression objective, we are excited to explore how it integrates with trajectory-straightening techniques, such as Shortcut Models and Mean Flow \citep{frans2025one, geng2026mean}.

\clearpage

\section{Acknowledgments} We thank Pascal Chang for helpful discussions and feedback. We also acknowledge the support of ETH Zurich in providing access to the Euler cluster for this research.

{
    \small
    \bibliographystyle{ieeenat_fullname}
    \bibliography{main}
}

\clearpage
\onecolumn
\begin{center}
    {\Large \textbf{\thetitle}}\\[1.0em]
    {\Large Supplementary Material}\\[2.0em]
\end{center}

\noindent We provide the following supplementary sections:
\begin{itemize}[leftmargin=1.8em, label=\small$\bullet$, itemsep=0.25em, topsep=0.4em, before=\small]
    \item \textbf{Section \ref{app:A}:} Proofs and theoretical derivations.
    \item \textbf{Section \ref{app:implementation_details}:} Implementation details, algorithm pseudocode, and hyperparameter settings.
    \item \textbf{Section \ref{sec:further_analysis}:} Further empirical analyses and theoretical discussions.
    \item \textbf{Section \ref{sec:few_step_quality}:} Comparison of few-step generation performance.
    \item \textbf{Section \ref{app:E}:} Additional qualitative results and generated visual samples.
    \item \textbf{Section \ref{app:F}:} Statement on the use of large language models.
\end{itemize}

\appendix
\setcounter{page}{1}

\section{Proofs}
\label{app:A}

\subsection{Angular Entropy via von Mises-Fisher KDE}
\label{app:angular_entropy}

As demonstrated by \citet{wang2020understanding}, minimizing the noise-noise repulsion term $\mathcal{L}_{\text{entropy}}$ is mathematically equivalent to maximizing the empirical differential entropy of the angular distribution, $H(\hat{Z})$. 

For an empirical batch of normalized vectors $\{\hat{z}_1, \dots, \hat{z}_N\}$, the empirical plug-in entropy is estimated via the sample average of the negative log-likelihood:
\begin{equation}
\hat{H}(\hat{Z}) = - \frac{1}{N} \sum_{i=1}^{N} \log \hat{\mu}(\hat{z}_i)
\label{eq:empirical_angular}
\end{equation}
where $\hat{\mu}(\hat{z}_i)$ is the estimated angular density at point $\hat{z}_i$. To evaluate this density cleanly on the unit hypersphere $\mathcal{S}^{d-1}$, we employ a \textit{von Mises-Fisher (vMF) kernel density estimator (KDE)} with a concentration parameter $\kappa = 1/\gamma$:
\begin{equation}
\hat{\mu}(\hat{z}_i) = \frac{1}{C_{\text{vMF}}(\gamma, d) (N-1)} \sum_{j \neq i} \exp\left( \frac{\langle \hat{z}_i, \hat{z}_j \rangle}{\gamma} \right)
\label{eq:vmf_kde}
\end{equation}
where $C_{\text{vMF}}(\gamma, d)$ is the partition function (inverse normalization constant) for the vMF distribution in $d$ dimensions. Substituting Equation \ref{eq:vmf_kde} into Equation \ref{eq:empirical_angular} yields:
\begin{align}
\hat{H}(\hat{Z}) &= - \frac{1}{N} \sum_{i=1}^{N} \log \left( \frac{1}{C_{\text{vMF}}(\gamma, d) (N-1)} \sum_{j \neq i} \exp\left( \frac{\langle \hat{z}_i, \hat{z}_j \rangle}{\gamma} \right) \right) \nonumber \\
&= - \left( \frac{1}{N} \sum_{i=1}^{N} \log \sum_{j \neq i} \exp\left( \frac{\langle \hat{z}_i, \hat{z}_j \rangle}{\gamma} \right) \right) + \log(C_{\text{vMF}}(\gamma, d)) + \log(N-1)
\label{eq:vmf_substitution}
\end{align}

We recognize the first expectation term on the right-hand side as the exact negative of our cross-sample particle repulsion loss, $\mathcal{L}_{\text{entropy}}$. Defining the constant $C = \log C_{\text{vMF}}(\gamma, d) + \log (N-1)$, we obtain the direct relation:
\begin{equation}
\mathcal{L}_{\text{entropy}} = - \hat{H}(\hat{Z}) + C
\label{eq:uniformity_entropy_relation}
\end{equation}
Thus, minimizing the cross-sample particle interaction $\mathcal{L}_{\text{entropy}}$ is mathematically equivalent to maximizing the empirical alternative of the angular differential entropy $\hat{H}(\hat{Z})$.

\subsection{Norm Growth during Discrete Noise Optimization}
\label{app:norm_growth}

In this section, we analyze the idealized baseline case of discrete optimization under vanilla gradient descent (GD, $\eta > 0$) with a scale-agnostic loss $\mathcal{L}(\hat{z})$ to demonstrate why unnormalized noise particles inherently expand outward. This centrifugal growth directly motivates our norm regularizer $\mathcal{L}_{\text{norm}}$. The mathematical foundation for this phenomenon was established by \citet{salimans2016weight, wang2017normface}; here, we adapt to our specific noise optimization setting.

Let $z \in \mathbb{R}^d$ be an unnormalized noise particle, and let $\hat{z}$ be its $L_2$-normalized counterpart on the unit hypersphere $\mathcal{S}^{d-1}$. Suppose that we are minimizing a  \textit{scale-agnostic} loss $\mathcal{L}(\hat{z})$. The gradient $g_t = \nabla_{z_t} \mathcal{L}$ is computed via the multi-variable chain rule:
\begin{equation}
    g_t = \left( \frac{\partial \hat{z}_t}{\partial z_t} \right)^T \nabla_{\hat{z}_t} \mathcal{L}
\end{equation}
Evaluating this derivative:
\begin{align}
    \frac{\partial \hat{z}}{\partial z} &= \frac{\partial}{\partial z} \left( z (z^T z)^{-\frac{1}{2}} \right) \nonumber \\
    &= I (z^T z)^{-\frac{1}{2}} - z (z^T z)^{-\frac{3}{2}} z^T \nonumber \\
    &= \frac{1}{\|z\|_2} \left( I - \frac{z z^T}{\|z\|_2^2} \right) \nonumber \\
    &= \frac{1}{\|z\|_2} \left( I - \hat{z} \hat{z}^T \right)
\end{align}
Substituting back into the gradient formula:
\begin{equation}
    g_t = \frac{1}{\|z_t\|_2} \left( I - \hat{z}_t \hat{z}_t^T \right) \nabla_{\hat{z}_t} \mathcal{L}
\end{equation}
This gradient $g_t$ is strictly orthogonal to the current unnormalized noise vector $z_t$. We verify this identity by taking their inner product:
\begin{align}
    z_t^T g_t &= z_t^T \left[ \frac{1}{\|z_t\|_2} \left( I - \hat{z}_t \hat{z}_t^T \right) \nabla_{\hat{z}_t} \mathcal{L} \right] \nonumber \\
    &= \frac{1}{\|z_t\|_2} \left( z_t^T - z_t^T \frac{z_t z_t^T}{\|z_t\|_2^2} \right) \nabla_{\hat{z}_t} \mathcal{L} \nonumber \\
    &= \frac{1}{\|z_t\|_2} \left( z_t^T - z_t^T \right) \nabla_{\hat{z}_t} \mathcal{L} = 0
\end{align}
Because $z_t^T g_t = 0$, every update step operates purely tangent to the hypersphere, seeking only to rotate the noise particles. Now, consider a discrete optimization update step with an optimization learning rate $\eta > 0$:
\begin{equation}
    z_{t+1} = z_t - \eta g_t
\end{equation}
We compute the squared $L_2$ norm of the updated particle $z_{t+1}$:
\begin{align}
    \|z_{t+1}\|_2^2 &= (z_t - \eta g_t)^T (z_t - \eta g_t) \nonumber \\
    &= \|z_t\|_2^2 - 2\eta z_t^T g_t + \eta^2 \|g_t\|_2^2
\end{align}
Exploiting ($z_t^T g_t = 0$), the cross-term vanishes, leaving:
\begin{equation}
    \|z_{t+1}\|_2^2 = \|z_t\|_2^2 + \eta^2 \|g_t\|_2^2 \geq \|z_t\|_2^2
\end{equation}

\subsection{Proof of Theorem \ref{thm:asymptotic_decomp}}
\label{sec:proof_theorem_1}

\begin{proof}
Let our combined objective be defined as $\mathcal{L}_{\text{total}} \triangleq \mathcal{L}_{\text{align}, \tau} + \beta \mathcal{L}_{\text{entropy}, \gamma} + \lambda \mathcal{L}_{\text{norm}}$. We extend the proof of Theorem 1 from \cite{wang2020understanding}.

By the Law of Large Numbers and the Continuous Mapping Theorem, as $N \to \infty$:
\begin{align}
    \lim_{N \to \infty} \left( \mathcal{L}_{\text{align}, \tau} - \log N \right) &= \lim_{N \to \infty} \left[ -\frac{1}{N} \sum_{i=1}^N \langle \hat{z}_i, \hat{x}_i \rangle / \tau + \frac{1}{N} \sum_{i=1}^N \log  \sum_{j=1}^N \exp(\langle \hat{z}_i, \hat{x}_j \rangle / \tau) - \log N \right] \nonumber\\ 
    &= \lim_{N \to \infty} \left[ -\frac{1}{N} \sum_{i=1}^N \langle \hat{z}_i, \hat{x}_i \rangle / \tau + \frac{1}{N} \sum_{i=1}^N \log \frac{1}{N} \sum_{j=1}^N \exp(\langle \hat{z}_i, \hat{x}_j \rangle / \tau) \right] \nonumber \\ 
    &= - \frac{1}{\tau} \mathbb{E}_{(z,x)\sim\pi} [\langle \hat{z}, \hat{x} \rangle] + \mathbb{E}_{z\sim\tilde{\mu}_0} \left[ \log \mathbb{E}_{x\sim\mu_1} \left[ \exp(\langle \hat{z}, \hat{x} \rangle / \tau) \right] \right] \nonumber
\end{align}
defining the second term as the cross-modal repulsion $\Phi_\tau(\tilde{\mu}_0, \mu_1)$.

We apply the identical logic to our empirical entropy loss $\mathcal{L}_{\text{entropy}, \gamma}$, which computes pairwise similarities strictly among $N$ noise particles drawn from $\tilde{\mu}_0$:
\begin{align}
    \lim_{N \to \infty} \left( \mathcal{L}_{\text{entropy}, \gamma} - \log N \right) 
    &= \lim_{N \to \infty} \left[ \frac{1}{N} \sum_{i=1}^N \log \left( \sum_{j \neq i} \exp(\langle \hat{z}_i, \hat{z}_j \rangle / \gamma) \right) - \log N \right] \nonumber\\
    &= \lim_{N \to \infty} \frac{1}{N} \sum_{i=1}^N \log \left( \frac{1}{N} \sum_{j \neq i} \exp(\langle \hat{z}_i, \hat{z}_j \rangle / \gamma) \right) \nonumber\\
    &= \lim_{N \to \infty} \frac{1}{N} \sum_{i=1}^N \log \left( \frac{N-1}{N} \cdot \frac{1}{N-1} \sum_{j \neq i} \exp(\langle \hat{z}_i, \hat{z}_j \rangle / \gamma) \right) \nonumber\\
    &= \underbrace{\lim_{N \to \infty} \log\left(\frac{N-1}{N}\right)}_{= 0} + \lim_{N \to \infty} \frac{1}{N} \sum_{i=1}^N \log \left( \frac{1}{N-1} \sum_{j \neq i} \exp(\langle \hat{z}_i, \hat{z}_j \rangle / \gamma) \right) \nonumber\\
    &= \mathbb{E}_{z\sim\tilde{\mu}_0} \left[ \log \mathbb{E}_{z'\sim\tilde{\mu}_0} \left[ \exp(\langle \hat{z}, \hat{z}' \rangle / \gamma) \right] \right] \triangleq \Phi_\gamma(\tilde{\mu}_0) \nonumber
\end{align}

Summing these respective limits and scaling the unimodal repulsion explicitly by $\beta$:
\begin{align}
    \lim_{N \to \infty} \left( \mathcal{L}_{\text{total}} - (1 + \beta)\log N \right) &= - \frac{1}{\tau} \mathbb{E}_{(z,x)\sim\pi} [\langle \hat{z}, \hat{x} \rangle] + \Phi_\tau(\tilde{\mu}_0, \mu_1) + \beta \Phi_\gamma(\tilde{\mu}_0) + \lambda\mathbb{E}_{z\sim\tilde{\mu}_0}\left[\|z\|^2\right] \nonumber
\end{align}
    
where we used the fact that $\lim_{N \to \infty} \mathcal{L}_{\text{norm}} = \lim_{N \to \infty} \frac{1}{N} \sum_{i=1}^{N} \|z_i\|_2^2 = \mathbb{E}_{z\sim\tilde{\mu}_0}\left[\|z\|^2\right]$. This proves Theorem \ref{thm:asymptotic_decomp}.
\end{proof}

\subsection{Proof of Corollary \ref{cor:induced_gaussianity}}
\label{sec:proof_corollary_1}
\begin{proof}
We adapt the proof from (\cite{betser2026infonce}, Appendix C). Let $z \in \mathbb{R}^d$ denote the unnormalized noise representation and write its polar decomposition as $z = r u$ with $r = \|z\|_2 > 0$ and $u := z/\|z\|_2 \in \mathcal{S}^{d-1}$. For any fixed $k \ge 1$, let $P_k$ be the corresponding orthogonal projector, setting $z_k := P_k z$ and $u_k := P_k u$. 

As $\beta \to \infty$, the asymptotic objective reduces entirely to minimizing the uniformity potential $\Phi_\gamma(\tilde{\mu}_0)$ (if $\lambda$ = 0). By \citet{wang2020understanding} (Appendix A), $\Phi_\gamma$ is uniquely minimized at the uniform law $\sigma$ on $\mathcal{S}^{d-1}$, hence the angular component satisfies $u \sim \sigma$. By the Maxwell-Poincaré spherical CLT (Lemma \ref{lem:maxwell_poincare}):
\begin{equation}
    \sqrt{d} u_k \Rightarrow \mathcal{N}(0, \mathbf{I}_k) \quad (d \to \infty).
\end{equation}
Because the scale-invariant objective evaluates strictly $L_2$ 
normalized vectors, its continuous gradient flow acts purely tangentially to the hypersphere (see Appendix \ref{app:norm_growth}). Under this idealized continuous optimization (and, i.e., $\lambda = 0$), the radial distribution of the standard Gaussian initialization $\mu_0 = \mathcal{N}(0, \mathbf{I})$ is preserved. By high-dimensional gaussian thin-shell concentration, the normalized radius converges in probability:
\begin{equation}
    \frac{r}{\sqrt{d}} \xrightarrow{P} 1 \quad (d \to \infty).
\end{equation}Since $z_k = r u_k = \left(\frac{r}{\sqrt{d}}\right) (\sqrt{d} u_k)$, combining the limits via Slutsky's theorem \cite{van2000asymptotic} results in:
\begin{equation}
    z_k \Rightarrow 1 \cdot \mathcal{N}(0, \mathbf{I}_k) = \mathcal{N}(0, \mathbf{I}_k) \quad (d \to \infty).
\end{equation}
which concludes the proof of Corollary \ref{cor:induced_gaussianity}. 
\end{proof}
\section{Implementation Details}
\label{app:implementation_details}

This section details the model architectures, training hyperparameters, and evaluation metrics used in our experiments.

\subsection{Evaluation Metrics} \label{appendix:metrics}

\paragraph{Fr\'{e}chet Inception Distance (FID). }
FID evaluates sample quality by comparing the empirical statistics of real and generated images in the feature space of a pre-trained Inception-v3 network \cite{heusel2017gans}. To match the baselines from other papers, we use the standard \texttt{clean-fid} implementation in legacy TensorFlow mode. Reference statistics are calculated on the training images and compared against 50,000 samples generated by the model.

\paragraph{Mean Path Curvature.} 
We apply the trajectory curvature metric introduced by \citet{lee2023minimizing}. Let $v_t(x_t)$ denote the predicted network velocity along the path from initial noise $x_0$ to the generated data $x_1$. The mean path curvature $\kappa$ is defined as:
\begin{equation}
    \kappa = \mathbb{E}_{t, x_0} \left[ \|(x_1 - x_0) - v_t(x_t)\|_2^2 \right]
\end{equation}
where $t \sim U(0, 1)$ and $x_0 \sim \mu_0 = \mathcal{N}(0, \mathbf{I})$. Curvature is not comparable across solvers and timestep schedules. However, within a fixed setup, lower values indicate straighter generative paths.

\subsection{Pseudocode}
The detailed algorithm for our noise optimization method is provided in Algorithm \ref{alg:npo_training}.

\begin{algorithm}[htbp]
\caption{Training Flow Matching with Contrastive Noise Alignment (CNA)}
\label{alg:npo_training}
\textbf{Inputs:} Flow model $v_\theta$, batch size $B$, initial coupling $\pi_{\text{init}}\in \Pi(\mu_0, \mu_1)$, CNA learning rate $\eta$, CNA optimization steps $T$, temperatures $\tau, \gamma$, loss weights $\beta, \lambda$ \\
\textbf{Outputs:} Trained network parameters $\theta$
\begin{algorithmic}[1]
\WHILE{network $\theta$ has not converged}
    \STATE Sample pairs $\{(z_i, x_i)\}_{i=1}^B \sim \pi_{\text{init}}$

    \STATE \textcolor{gray}{\textit{\# Contrastive Noise Alignment (CNA)}}
    \STATE $\hat{X} \leftarrow \{ x_i / \|x_i\|_2 \}_{i=1}^B$ 
    \FOR{$t = 1$ \TO $T$}
        \STATE $\hat{Z} \leftarrow \{ z_i / \|z_i\|_2 \}_{i=1}^B$ 
        \STATE $\mathcal{L}_{\text{align}} := \text{InfoNCE}(\hat{Z}, \hat{X}, \tau)$
        \STATE $\mathcal{L}_{\text{entropy}} := \frac{1}{B} \sum_{i=1}^{B} \log \left( \sum_{j \neq i} \exp \left( \frac{\hat{z}_i \cdot \hat{z}_j}{\gamma} \right) \right)$
        \STATE $\mathcal{L}_{\text{norm}} := \frac{1}{B} \sum_{i=1}^{B} \|z_i\|_2^2$
        \STATE $\mathcal{L}_{\text{CNA}} := \mathcal{L}_{\text{align}} + \beta \mathcal{L}_{\text{entropy}} + \lambda \mathcal{L}_{\text{norm}}$
        \STATE $Z \leftarrow Z - \eta \cdot \nabla_{Z} \mathcal{L}_{\text{CNA}}$ \hfill \COMMENT{\textit{$\triangleright$ Update noise states}}
    \ENDFOR
    \STATE Let $Z^* = \{z_i^*\}_{i=1}^B \leftarrow Z$ be the optimized coupled noise batch

    \STATE $\theta \leftarrow \text{FMStep}(\theta, (z_i^*)_{i=1}^B, (x_i)_{i=1}^B)$ \hfill \COMMENT{\textit{$\triangleright$ Sample $t$, compute vector field loss, update $\theta$}}
\ENDWHILE
\STATE \textbf{return} $\theta$
\end{algorithmic}
\end{algorithm}

\subsection{Flow Model Training}
\label{sec:flow_model_training}

For each dataset, the specific U-Net architectural choices and hyperparameters are detailed in Table \ref{tab:adm_params}.

\begin{table}[h]
\centering
\caption{ADM network architecture and hyperparameters used for flow model training.}
\label{tab:adm_params}
\renewcommand{\arraystretch}{1.1} 
\begin{tabular}{@{}lcc@{}}
\toprule
 & \textbf{CIFAR-10} & \textbf{ImageNet (32x32) } \\
\midrule
Channels & 128 & 256 \\
Depth & 2 & 3 \\
Channels multipliers & $1, 2, 2, 2$ & $1, 2, 2, 2$ \\
Heads & 4 & 4 \\
Heads channels & 64 & 64 \\
Attention resolution & $16$ & $4$ \\
Dropout & 0.1 & 0.0 \\
Batch size / GPU & 128 & 128 \\
Effective batch size & 128 & 1024 \\
GPUs & 1 & 8 \\
Iterations & 390k & 400k \\
Learning rate & 0.0002 & 0.0001 \\
Learning rate scheduler & Constant & Polynomial Decay \\
Warmup steps & 5k & 20k \\
Model parameters & 36M & 189M \\
\bottomrule
\end{tabular}
\end{table}

\textbf{CIFAR-10.} To ensure a fair comparison with prior work, we adopt the exact ADM architecture and hyperparameters specified by \citet{tong_improving_2024}. For this dataset, all models are trained using a single NVIDIA RTX 4090 GPU.

\textbf{ImageNet.} We follow the U-Net configuration and hyperparameters described by \citet{pooladian_multisample_2023}. These models are trained on a single node with $8 \times$ NVIDIA RTX 4090 GPUs.

For all models, we also apply an Exponential Moving Average (EMA) with a decay factor of $0.9999$ to the weights.

\subsection{Hyperparameter Settings}
\label{sec:hyperparameters}

\begin{table}[htbp]
\centering
\caption{Hyperparameter configurations for CNA and CNA + OT. Shared optimization settings are applied across both strategies and datasets.}
\label{tab:latent_optimization_hyperparams}

\begin{subtable}[t]{0.34\textwidth}
\centering
\caption{\textit{Shared} Parameters}
\label{tab:shared_hyperparams}
\small
\begin{tabular}{@{}lc@{}}
\toprule
Parameter & \quad \\
\midrule
\rowcolor{gray!10} 
Optimization steps $N_{\text{opt}}$ & $4$ \\
Learning rate $\eta$ & $0.02$ \\
$\tau$ & $0.01$ \\
$\gamma$ & $0.01$ \\
\bottomrule
\end{tabular}
\end{subtable}%
\hfill 
\begin{subtable}[t]{0.64\textwidth}
\centering
\caption{\textit{Tuned} Parameters}
\label{tab:dataset_hyperparams}
\small
\begin{tabular}{@{}lcccc@{}}
\toprule
\multirow{2}{*}{Dataset} & \multicolumn{2}{c}{CNA} & \multicolumn{2}{c}{CNA + OT} \\
\cmidrule(lr){2-3} \cmidrule(l){4-5}
 & $\beta$ & $\lambda$ & $\beta$ & $\lambda$ \\
\midrule
CIFAR-10   & 1-5 & 0.001 & 1-10 & 0.001 \\
ImageNet32 & 1-5 & 0.001 & 1-50 & \{0.001, 0.005\} \\
\bottomrule
\end{tabular}
\end{subtable}

\end{table}

For our main experiments, we evaluate our methods on CIFAR-10 and ImageNet ($32 \times 32$). The complete configuration of optimization parameters, temperatures, and loss weights is detailed in Table \ref{tab:latent_optimization_hyperparams}. 

For both datasets, the number of optimization steps is set to $N_{\text{opt}} = 4$ using \textit{Adam} optimizer with default settings and a learning rate of $\eta = 0.02$. The temperature parameters controlling feature scaling and field repulsion are fixed at $\tau = 0.01$ and $\gamma = 0.01$, respectively. Furthermore, the norm regularization weight $\lambda$ is tuned such that the average $L_2$ norm of the optimized batch is preserved. We evaluate two initialization strategies: standard independent coupling (CNA) and pre-structured initialization via minibatch OT (CNA + OT) \citep{pooladian_multisample_2023, tong_improving_2024}.

\section{Further Analyses And Discussion}
\label{sec:further_analysis}

\subsection{Connection to Kullback-Leibler Minimization}
\label{app:kl_connection}
To understand what our regularizer actually targets, we can connect it to the Kullback-Leibler (KL) divergence between our optimized prior $q(z)$ and the target standard Gaussian $\mathcal{N}(0, \mathbf{I})$:
\begin{equation}
    D_{\mathrm{KL}}(q \parallel \mathcal{N}(0, \mathbf{I})) = -h(q) + \frac{1}{2} \mathbb{E}_{z \sim q}\left[ \|z\|_2^2 \right] + \text{const.}
\end{equation}
where $h(q)$ is the joint differential entropy. Decomposing $z$ into polar coordinates $z = r u$ (where $r = \|z\|_2$ and $u = z/\|z\|_2 \in \mathcal{S}^{d-1}$), this joint entropy expands as:
\begin{equation}
    h(q) = h(u) + h(r \mid u) + (d-1)\mathbb{E}[\log r]
\end{equation}
where $h(u)$ is the angular entropy on the hypersphere and $h(r \mid u)$ is the conditional radial entropy.

Comparing this decomposition to our regularizer $L_{\text{reg}} = \beta \mathcal{L}_{\text{entropy}} + \lambda \mathcal{L}_{\text{norm}}$, we see that while we optimize the angular entropy $h(u)$ (via $\mathcal{L}_{\text{entropy}}$) and penalize the second moment (via $\mathcal{L}_{\text{norm}}$), our objective \emph{lacks} the conditional radial entropy $h(r \mid u)$. 

Thus, our method does not minimize the full joint KL divergence; instead, it is best understood as a proxy that approximates the angular part of the KL penalty, relying on $\mathcal{L}_{\text{norm}}$ purely as an empirical scale stabilizer (see Appendix \ref{app:norm_growth}).

\subsection{Empirical Verification of Gaussianity}
To verify that CNA preserves the structure of the Gaussian prior during training, we investigate the optimized noise across two key geometric properties: radial concentration and spatial isotropy.
\paragraph{Radial concentration.} We analyze the radial concentration of noise particles to verify whether the optimization preserves the $L_2$ scale of the Gaussian prior. In high dimensions, a standard Gaussian vector $z \sim \mathcal{N}(0, \mathbf{I}_d)$ concentrates on a thin spherical shell following a Chi distribution $\chi_d$. Figure \ref{fig:radial_concentration} compares the norm density of the initial noise against optimized noise with and without the radial norm penalty $\mathcal{L}_{\text{norm}}$. Because cosine alignment ($\mathcal{L}_{\text{align}}$) and angular entropy ($\mathcal{L}_{\text{entropy}}$) operate on normalized vectors, optimizing without $\mathcal{L}_{\text{norm}}$ causes norm expansion, shifting the density rightward. In contrast, adding our radial norm penalty ($\lambda = 0.001$, for $\beta = 5$) counteracts this drift, keeping the spatial distribution virtually unchanged.
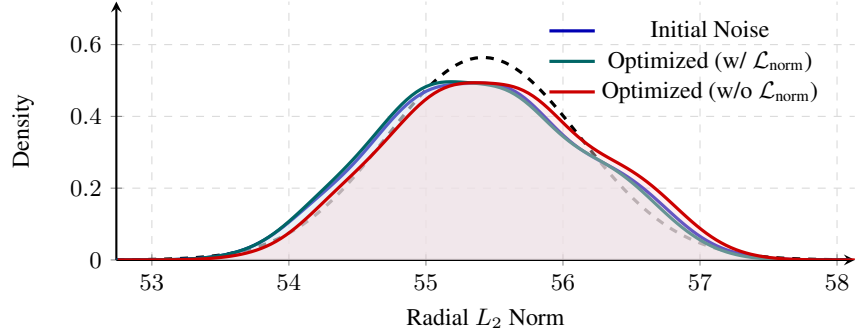
\begin{figure}[ht]
    \centering
    \small
    
    \pgfplotstableread{
    Norm InitKDE OptKDE ChiPDF
    52.74078 0.00001 0.00001 0.00039
    52.77669 0.00002 0.00002 0.00047
    52.81260 0.00003 0.00003 0.00057
    52.84850 0.00004 0.00004 0.00069
    52.88441 0.00007 0.00007 0.00084
    52.92032 0.00010 0.00010 0.00101
    52.95623 0.00015 0.00015 0.00121
    52.99213 0.00022 0.00023 0.00145
    53.02804 0.00033 0.00033 0.00173
    53.06395 0.00047 0.00048 0.00205
    53.09985 0.00066 0.00068 0.00244
    53.13576 0.00091 0.00094 0.00288
    53.17167 0.00124 0.00128 0.00341
    53.20757 0.00167 0.00173 0.00401
    53.24348 0.00222 0.00230 0.00471
    53.27939 0.00291 0.00301 0.00551
    53.31530 0.00376 0.00389 0.00644
    53.35120 0.00480 0.00496 0.00750
    53.38711 0.00606 0.00625 0.00871
    53.42302 0.00757 0.00779 0.01009
    53.45892 0.00937 0.00963 0.01166
    53.49483 0.01151 0.01180 0.01344
    53.53074 0.01403 0.01434 0.01545
    53.56664 0.01699 0.01731 0.01771
    53.60255 0.02044 0.02077 0.02024
    53.63846 0.02446 0.02479 0.02308
    53.67437 0.02910 0.02944 0.02624
    53.71027 0.03444 0.03478 0.02976
    53.74618 0.04052 0.04088 0.03366
    53.78209 0.04739 0.04779 0.03797
    53.81799 0.05506 0.05555 0.04272
    53.85390 0.06353 0.06416 0.04794
    53.88980 0.07277 0.07362 0.05365
    53.92571 0.08272 0.08387 0.05988
    53.96162 0.09328 0.09484 0.06666
    53.99753 0.10434 0.10643 0.07400
    54.03343 0.11578 0.11853 0.08194
    54.06934 0.12748 0.13101 0.09050
    54.10525 0.13934 0.14376 0.09968
    54.14116 0.15126 0.15669 0.10950
    54.17706 0.16320 0.16970 0.11998
    54.21297 0.17515 0.18277 0.13111
    54.24887 0.18716 0.19589 0.14289
    54.28478 0.19927 0.20908 0.15533
    54.32069 0.21161 0.22240 0.16840
    54.35660 0.22426 0.23592 0.18210
    54.39250 0.23735 0.24971 0.19639
    54.42841 0.25095 0.26385 0.21124
    54.46432 0.26514 0.27840 0.22663
    54.50023 0.27991 0.29337 0.24249
    54.53613 0.29524 0.30876 0.25879
    54.57204 0.31103 0.32452 0.27546
    54.60794 0.32716 0.34056 0.29244
    54.64385 0.34346 0.35677 0.30966
    54.67976 0.35971 0.37296 0.32703
    54.71567 0.37571 0.38896 0.34447
    54.75157 0.39121 0.40453 0.36190
    54.78748 0.40600 0.41944 0.37922
    54.82339 0.41987 0.43346 0.39633
    54.85929 0.43263 0.44636 0.41314
    54.89520 0.44413 0.45794 0.42953
    54.93111 0.45427 0.46804 0.44542
    54.96701 0.46299 0.47654 0.46070
    55.00292 0.47028 0.48341 0.47526
    55.03883 0.47619 0.48867 0.48901
    55.07474 0.48082 0.49243 0.50185
    55.11064 0.48433 0.49485 0.51370
    55.14655 0.48689 0.49613 0.52446
    55.18246 0.48870 0.49652 0.53406
    55.21836 0.48997 0.49626 0.54244
    55.25427 0.49087 0.49556 0.54952
    55.29018 0.49153 0.49459 0.55525
    55.32608 0.49201 0.49346 0.55960
    55.36199 0.49233 0.49220 0.56252
    55.39790 0.49239 0.49076 0.56401
    55.43380 0.49206 0.48900 0.56404
    55.46971 0.49113 0.48676 0.56262
    55.50562 0.48937 0.48383 0.55975
    55.54153 0.48654 0.47997 0.55547
    55.57743 0.48243 0.47501 0.54981
    55.61334 0.47688 0.46878 0.54281
    55.64925 0.46977 0.46119 0.53452
    55.68515 0.46110 0.45225 0.52501
    55.72106 0.45094 0.44201 0.51434
    55.75697 0.43944 0.43065 0.50261
    55.79287 0.42684 0.41838 0.48988
    55.82878 0.41342 0.40547 0.47625
    55.86469 0.39951 0.39224 0.46182
    55.90060 0.38542 0.37899 0.44668
    55.93650 0.37148 0.36600 0.43094
    55.97241 0.35797 0.35353 0.41469
    56.00832 0.34511 0.34175 0.39803
    56.04422 0.33305 0.33078 0.38107
    56.08013 0.32188 0.32063 0.36390
    56.11604 0.31163 0.31128 0.34663
    56.15194 0.30223 0.30260 0.32933
    56.18785 0.29358 0.29444 0.31210
    56.22376 0.28552 0.28659 0.29502
    56.25967 0.27789 0.27884 0.27817
    56.29557 0.27049 0.27098 0.26161
    56.33148 0.26312 0.26281 0.24542
    56.36739 0.25560 0.25417 0.22965
    56.40329 0.24777 0.24496 0.21434
    56.43920 0.23950 0.23509 0.19955
    56.47511 0.23069 0.22456 0.18531
    56.51101 0.22127 0.21337 0.17165
    56.54692 0.21122 0.20161 0.15860
    56.58283 0.20053 0.18936 0.14616
    56.61874 0.18926 0.17676 0.13437
    56.65464 0.17749 0.16394 0.12321
    56.69055 0.16532 0.15105 0.11270
    56.72646 0.15289 0.13824 0.10282
    56.76236 0.14035 0.12566 0.09357
    56.79827 0.12787 0.11343 0.08494
    56.83418 0.11559 0.10168 0.07692
    56.87008 0.10369 0.09050 0.06947
    56.90599 0.09231 0.07996 0.06259
    56.94190 0.08155 0.07013 0.05625
    56.97781 0.07152 0.06104 0.05042
    57.01371 0.06228 0.05272 0.04509
    57.04962 0.05386 0.04516 0.04022
    57.08553 0.04628 0.03836 0.03578
    57.12143 0.03950 0.03230 0.03176
    57.15734 0.03351 0.02695 0.02811
    57.19324 0.02825 0.02227 0.02483
    57.22915 0.02366 0.01821 0.02187
    57.26506 0.01969 0.01474 0.01921
    57.30097 0.01627 0.01180 0.01684
    57.33687 0.01334 0.00933 0.01472
    57.37278 0.01086 0.00729 0.01284
    57.40869 0.00875 0.00563 0.01117
    57.44460 0.00699 0.00429 0.00969
    57.48050 0.00553 0.00322 0.00839
    57.51641 0.00432 0.00239 0.00725
    57.55231 0.00334 0.00174 0.00624
    57.58822 0.00254 0.00125 0.00536
    57.62413 0.00192 0.00089 0.00460
    57.66003 0.00142 0.00062 0.00393
    57.69594 0.00104 0.00043 0.00335
    57.73185 0.00075 0.00029 0.00285
    57.76776 0.00053 0.00019 0.00242
    57.80367 0.00037 0.00013 0.00205
    57.83957 0.00026 0.00008 0.00173
    57.87548 0.00017 0.00005 0.00146
    57.91138 0.00012 0.00003 0.00122
    57.94729 0.00008 0.00002 0.00103
    57.98320 0.00005 0.00001 0.00086
    58.01910 0.00003 0.00001 0.00072
    58.05501 0.00002 0.00000 0.00060
    58.09092 0.00001 0.00000 0.00049
    }{\tableWithNorm}

    \pgfplotstableread{
    Norm InitKDE OptKDE ChiPDF
    52.75058 0.00001 0.00000 0.00041
    52.78669 0.00002 0.00001 0.00050
    52.82278 0.00003 0.00001 0.00061
    52.85889 0.00005 0.00001 0.00073
    52.89499 0.00007 0.00002 0.00088
    52.93109 0.00011 0.00004 0.00106
    52.96719 0.00017 0.00006 0.00128
    53.00329 0.00025 0.00009 0.00153
    53.03939 0.00037 0.00014 0.00182
    53.07549 0.00052 0.00021 0.00217
    53.11159 0.00073 0.00031 0.00258
    53.14769 0.00101 0.00044 0.00305
    53.18379 0.00138 0.00063 0.00360
    53.21989 0.00185 0.00087 0.00424
    53.25599 0.00244 0.00119 0.00498
    53.29210 0.00319 0.00161 0.00583
    53.32819 0.00411 0.00213 0.00680
    53.36430 0.00523 0.00279 0.00792
    53.40039 0.00658 0.00360 0.00920
    53.43650 0.00821 0.00458 0.01066
    53.47260 0.01014 0.00577 0.01232
    53.50870 0.01244 0.00719 0.01419
    53.54480 0.01513 0.00887 0.01630
    53.58090 0.01829 0.01084 0.01868
    53.61700 0.02199 0.01316 0.02135
    53.65310 0.02627 0.01588 0.02433
    53.68920 0.03122 0.01904 0.02765
    53.72530 0.03689 0.02274 0.03135
    53.76140 0.04333 0.02702 0.03544
    53.79750 0.05058 0.03199 0.03996
    53.83360 0.05865 0.03771 0.04493
    53.86971 0.06751 0.04424 0.05039
    53.90580 0.07712 0.05165 0.05636
    53.94191 0.08741 0.05996 0.06287
    53.97801 0.09827 0.06918 0.06994
    54.01411 0.10958 0.07927 0.07760
    54.05021 0.12122 0.09018 0.08586
    54.08631 0.13307 0.10180 0.09476
    54.12241 0.14503 0.11401 0.10429
    54.15851 0.15703 0.12669 0.11448
    54.19461 0.16904 0.13968 0.12533
    54.23071 0.18108 0.15287 0.13685
    54.26681 0.19319 0.16612 0.14902
    54.30291 0.20547 0.17938 0.16185
    54.33901 0.21802 0.19259 0.17531
    54.37511 0.23095 0.20574 0.18939
    54.41122 0.24437 0.21887 0.20406
    54.44732 0.25835 0.23202 0.21928
    54.48342 0.27292 0.24528 0.23501
    54.51952 0.28808 0.25873 0.25120
    54.55562 0.30376 0.27243 0.26780
    54.59172 0.31984 0.28644 0.28474
    54.62782 0.33618 0.30079 0.30195
    54.66392 0.35256 0.31547 0.31935
    54.70002 0.36878 0.33043 0.33687
    54.73612 0.38461 0.34559 0.35441
    54.77222 0.39982 0.36083 0.37188
    54.80832 0.41417 0.37600 0.38918
    54.84442 0.42749 0.39093 0.40622
    54.88052 0.43959 0.40542 0.42289
    54.91663 0.45035 0.41927 0.43908
    54.95272 0.45969 0.43226 0.45470
    54.98883 0.46759 0.44421 0.46963
    55.02493 0.47406 0.45496 0.48379
    55.06103 0.47920 0.46437 0.49706
    55.09713 0.48313 0.47237 0.50936
    55.13323 0.48604 0.47893 0.52060
    55.16933 0.48811 0.48410 0.53069
    55.20543 0.48956 0.48796 0.53957
    55.24153 0.49087 0.49067 0.54716
    55.27763 0.49132 0.49241 0.55340
    55.31373 0.49186 0.49337 0.55826
    55.34983 0.49224 0.49374 0.56169
    55.38593 0.49241 0.49369 0.56367
    55.42204 0.49222 0.49332 0.56419
    55.45813 0.49151 0.49268 0.56323
    55.49424 0.49003 0.49175 0.56082
    55.53033 0.48755 0.49045 0.55696
    55.56644 0.48384 0.48862 0.55169
    55.60254 0.47871 0.48607 0.54505
    55.63864 0.47203 0.48260 0.53710
    55.67474 0.46377 0.47801 0.52789
    55.71084 0.45398 0.47214 0.51749
    55.74694 0.44277 0.46489 0.50599
    55.78304 0.43039 0.45623 0.49346
    55.81914 0.41709 0.44621 0.47999
    55.85524 0.40320 0.43497 0.46569
    55.89134 0.38905 0.42272 0.45065
    55.92744 0.37497 0.40972 0.43496
    55.96354 0.36125 0.39627 0.41874
    55.99965 0.34814 0.38269 0.40208
    56.03574 0.33582 0.36929 0.38510
    56.07185 0.32438 0.35633 0.36788
    56.10795 0.31386 0.34401 0.35052
    56.14405 0.30423 0.33248 0.33313
    56.18015 0.29538 0.32179 0.31579
    56.21625 0.28717 0.31195 0.29858
    56.25235 0.27942 0.30287 0.28158
    56.28845 0.27195 0.29442 0.26487
    56.32455 0.26454 0.28643 0.24851
    56.36065 0.25703 0.27871 0.23257
    56.39675 0.24923 0.27105 0.21709
    56.43285 0.24100 0.26326 0.20212
    56.46895 0.23224 0.25518 0.18771
    56.50505 0.22288 0.24666 0.17388
    56.54116 0.21287 0.23759 0.16065
    56.57726 0.20223 0.22791 0.14805
    56.61336 0.19098 0.21758 0.13609
    56.64946 0.17922 0.20663 0.12478
    56.68556 0.16703 0.19509 0.11412
    56.72166 0.15456 0.18304 0.10410
    56.75776 0.14196 0.17061 0.09472
    56.79386 0.12939 0.15791 0.08597
    56.82996 0.11702 0.14510 0.07783
    56.86606 0.10500 0.13233 0.07028
    56.90216 0.09349 0.11976 0.06330
    56.93826 0.08261 0.10754 0.05687
    56.97437 0.07245 0.09581 0.05096
    57.01046 0.06308 0.08468 0.04555
    57.04657 0.05455 0.07426 0.04061
    57.08266 0.04685 0.06461 0.03612
    57.11877 0.03998 0.05577 0.03204
    57.15487 0.03390 0.04776 0.02835
    57.19097 0.02856 0.04058 0.02502
    57.22707 0.02391 0.03421 0.02203
    57.26317 0.01989 0.02860 0.01935
    57.29927 0.01642 0.02371 0.01695
    57.33537 0.01346 0.01950 0.01481
    57.37147 0.01094 0.01588 0.01291
    57.40869 0.00881 0.01282 0.01122
    57.44367 0.00703 0.01024 0.00973
    57.47977 0.00555 0.00810 0.00842
    57.51587 0.00433 0.00634 0.00726
    57.55198 0.00334 0.00490 0.00625
    57.58807 0.00255 0.00374 0.00537
    57.62418 0.00191 0.00282 0.00459
    57.66028 0.00142 0.00210 0.00392
    57.69638 0.00104 0.00154 0.00334
    57.73248 0.00075 0.00112 0.00284
    57.76858 0.00053 0.00080 0.00241
    57.80468 0.00037 0.00056 0.00204
    57.84078 0.00025 0.00039 0.00172
    57.87688 0.00017 0.00026 0.00145
    57.91298 0.00011 0.00018 0.00121
    57.94908 0.00007 0.00012 0.00102
    57.98518 0.00005 0.00008 0.00085
    58.02128 0.00003 0.00005 0.00071
    58.05738 0.00002 0.00003 0.00059
    58.09348 0.00001 0.00002 0.00049
    58.12959 0.00001 0.00001 0.00040
    }{\tableWithoutNorm}

    \begin{tikzpicture}
        \begin{axis}[
            width=0.65\linewidth,
            height=5cm,
            xlabel={Radial $L_2$ Norm},
            ylabel={Density},
            grid=major,
            grid style={dashed, gray!30},
            legend pos=north east,
            legend style={draw=none, fill=none, font=\small},
            axis lines=left,
            thick,
            enlarge x limits=false,
            ymin=0, ymax=0.72
        ]
        
        \addplot[
            forget plot,
            color=black, 
            dashed, 
            line width=1.2pt,
            mark=none
        ] table [x=Norm, y=ChiPDF] {\tableWithNorm};

        \addplot[
            color=blue!70!black, 
            line width=1.2pt,
            mark=none,
            fill=blue!10, fill opacity=0.4
        ] table [x=Norm, y=InitKDE] {\tableWithNorm};
        \addlegendentry{Initial Noise}
        
        \addplot[
            color=teal!80!black, 
            line width=1.2pt,
            mark=none,
            fill=teal!15, fill opacity=0.4
        ] table [x=Norm, y=OptKDE] {\tableWithNorm};
        \addlegendentry{Optimized (w/ $\mathcal{L}_{\text{norm}}$)}

        \addplot[
            color=red!80!black, 
            line width=1.2pt,
            mark=none,
            fill=red!15, fill opacity=0.4
        ] table [x=Norm, y=OptKDE] {\tableWithoutNorm};
        \addlegendentry{Optimized (w/o $\mathcal{L}_{\text{norm}}$)}
        
        \end{axis}
    \end{tikzpicture}
    \caption{Radial concentration density of noise particles ($d=3072$). When trained without the radial norm penalty ($\mathcal{L}_{\text{norm}}$), scale-agnostic contrastive alignment induces centrifugal expansion, shifting particle norms rightward. Including $\mathcal{L}_{\text{norm}}$ successfully prevents this.}
    \label{fig:radial_concentration}
\end{figure}

\paragraph{Preserving isotropic structure.} We run Principal Component Analysis (PCA) on optimized noise batches to check whether CNA maintains isotropic variance across dimensions. Figure \ref{fig:eigen_decay_beta} shows the eigenvalue decay across regularization strengths ($\beta$). Without structural penalties ($\beta = 0$), the noise over-adapts and leading dimensions dominate. Introducing entropy and norm penalties ($\beta = 1, 5$) pulls the spectrum back to the baseline. At $\beta = 5$, the decay matches standard random noise, confirming that repulsive forces prevent spatial collapse while keeping the Gaussian prior intact.

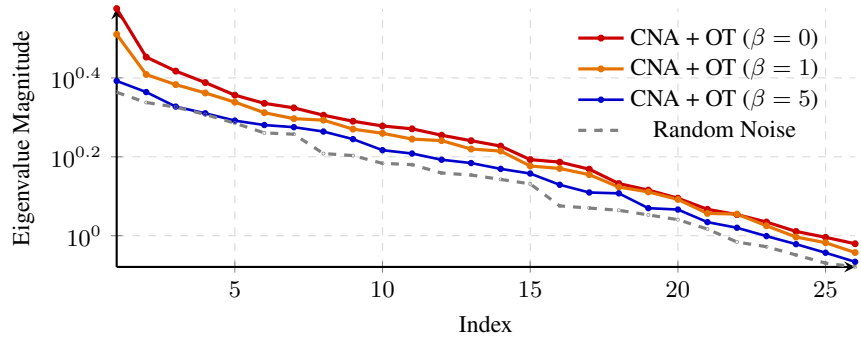
\begin{figure}[ht]
    \centering
    \small
    \pgfplotstableread{
    Component AlignedBeta0 AlignedBeta1 AlignedBeta5 Random
    1 3.764709e+00 3.239363e+00 2.467555e+00 2.309396e+00
    2 2.836128e+00 2.561717e+00 2.312198e+00 2.175973e+00
    3 2.613418e+00 2.414694e+00 2.123885e+00 2.117991e+00
    4 2.444073e+00 2.300291e+00 2.042991e+00 2.028759e+00
    5 2.271283e+00 2.180482e+00 1.957436e+00 1.928792e+00
    6 2.165148e+00 2.050123e+00 1.907780e+00 1.822239e+00
    7 2.108673e+00 1.980296e+00 1.884194e+00 1.808337e+00
    8 2.020113e+00 1.963181e+00 1.836680e+00 1.615476e+00
    9 1.950801e+00 1.863012e+00 1.757150e+00 1.596141e+00
    10 1.897334e+00 1.818162e+00 1.647620e+00 1.525682e+00
    11 1.866087e+00 1.757822e+00 1.615337e+00 1.513988e+00
    12 1.797639e+00 1.741931e+00 1.558016e+00 1.441855e+00
    13 1.741144e+00 1.658319e+00 1.528085e+00 1.424796e+00
    14 1.687026e+00 1.639158e+00 1.477062e+00 1.389154e+00
    15 1.559174e+00 1.502187e+00 1.437404e+00 1.353394e+00
    16 1.536640e+00 1.480643e+00 1.346126e+00 1.189998e+00
    17 1.474603e+00 1.428298e+00 1.286807e+00 1.175255e+00
    18 1.355619e+00 1.328471e+00 1.280617e+00 1.159913e+00
    19 1.304967e+00 1.291182e+00 1.174281e+00 1.129174e+00
    20 1.244721e+00 1.235914e+00 1.165079e+00 1.098236e+00
    21 1.166506e+00 1.139061e+00 1.082518e+00 1.039592e+00
    22 1.131366e+00 1.134502e+00 1.047557e+00 9.638383e-01
    23 1.083089e+00 1.059700e+00 9.976327e-01 9.378095e-01
    24 1.025379e+00 9.925826e-01 9.519389e-01 8.937466e-01
    25 9.907168e-01 9.599075e-01 9.053615e-01 8.522193e-01
    26 9.542714e-01 9.065035e-01 8.588039e-01 8.334159e-01
    27 8.661701e-01 8.546744e-01 8.286870e-01 7.755365e-01
    28 8.625995e-01 8.446990e-01 8.061573e-01 7.605220e-01
    29 8.263313e-01 8.098547e-01 7.583583e-01 7.352925e-01
    30 7.697717e-01 7.502446e-01 6.974233e-01 6.945721e-01
    31 7.231128e-01 7.120427e-01 6.744737e-01 6.652944e-01
    32 7.109740e-01 6.901852e-01 6.323551e-01 6.333215e-01
    33 6.689039e-01 6.648245e-01 6.081355e-01 5.940924e-01
    34 6.620104e-01 6.290888e-01 5.696052e-01 5.600788e-01
    35 6.081394e-01 6.091845e-01 5.537072e-01 5.293010e-01
    36 5.582021e-01 5.390331e-01 5.428711e-01 4.938301e-01
    37 5.319616e-01 5.295989e-01 5.225269e-01 4.795038e-01
    38 5.014983e-01 5.006438e-01 4.516244e-01 4.392713e-01
    39 4.420860e-01 4.347481e-01 4.333064e-01 4.187873e-01
    40 4.114169e-01 4.181080e-01 4.021367e-01 3.979616e-01
    41 3.874238e-01 3.898705e-01 3.828558e-01 3.690529e-01
    42 3.486825e-01 3.533748e-01 3.434483e-01 3.621407e-01
    43 3.360963e-01 3.342303e-01 3.323025e-01 3.413664e-01
    44 3.165981e-01 3.117444e-01 3.055158e-01 2.837513e-01
    45 2.771373e-01 2.722387e-01 2.561093e-01 2.629166e-01
    46 2.145678e-01 2.205574e-01 2.196758e-01 2.331539e-01
    47 2.053996e-01 2.061221e-01 2.008128e-01 2.193117e-01
    48 1.849604e-01 1.829260e-01 1.674118e-01 1.724306e-01
    }\datatable

    \begin{tikzpicture}
        \begin{axis}[
            width=0.65\linewidth,
            height=5cm,
            ymode=log, 
            xlabel={Index},
            ylabel={Eigenvalue Magnitude},
            grid=major,
            grid style={dashed, gray!30},
            legend pos=north east,
            legend style={draw=none, fill=none, font=\small},
            axis lines=left,
            thick,
            xmin=1, xmax=26,
        ]
        
        \addplot[
            color=red!80!black, 
            mark=*, 
            mark size=0.7pt,
            line width=1.2pt
        ] table [x=Component, y=AlignedBeta0] {\datatable};
        \addlegendentry{CNA + OT ($\beta = 0$)}

        \addplot[
            color=orange!90!black, 
            mark=*, 
            mark size=0.7pt,
            line width=1.2pt
        ] table [x=Component, y=AlignedBeta1] {\datatable};
        \addlegendentry{CNA + OT ($\beta = 1$)}
        
        \addplot[
            color=blue!80!black, 
            mark=*, 
            mark size=0.7pt,
            line width=1pt
        ] table [x=Component, y=AlignedBeta5] {\datatable};
        \addlegendentry{CNA + OT ($\beta = 5$)}
        
        \addplot[
            color=gray, 
            dashed, 
            mark=*, 
            mark size=0.1pt,
            mark options={solid}, 
            line width=1.2pt
        ] table [x=Component, y=Random] {\datatable};
        \addlegendentry{Random Noise}
        
        \end{axis}
    \end{tikzpicture}
    \caption{Eigenvalue decay (log scale) of the source noise distributions across different regularization strengths.}
    \label{fig:eigen_decay_beta}
\end{figure}

\subsection{Computational Efficiency}
\label{sec:comp_efficiency}
We analyze the scalability of our method via training throughput as shown in Table \ref{tab:runtime_comparison}. CNA introduces minimal overhead relative to independent couplings. Because the inner optimization converges in just a few steps, the actual time penalty is practically negligible, costing less than a ${\sim}2\%$ drop in throughput. Even at larger batch sizes ($N=1024$) where the exact OT solver becomes a noticeable bottleneck, our contrastive updates scale effortlessly without adding any extra drag.

\begin{table}[htbp]
    \centering
    \small
    \caption{\textbf{Training throughput comparison} (Iterations / s, measured on Nvidia RTX 4090).}
    \label{tab:runtime_comparison}
    \begin{tabular}{@{} l cccc @{}}
    \toprule
    Dataset & I-CFM & CNA (ours) & OT-CFM & CNA + OT (ours) \\
    \midrule
    CIFAR-10 ($N=128$)    & 6.68 & 6.58 & 6.65 & 6.56 \\
    ImageNet32 ($N=1024$) & 1.77 & 1.75 & 1.31 & 1.30 \\
    \bottomrule
    \end{tabular}
\end{table}





\subsection{Post-Hoc Path Regulation via \texorpdfstring{$\beta$}{Beta}-Conditioning}
\label{sec:beta_conditioning_appendix}

We train an adaptive variant of our model where the entropy regularization weight $\beta$ is sampled uniformly from the interval $[1, 6]$ during training. To handle this continuous conditioning, we compute a sinusoidal embedding of $\beta$ passed through a 2-layer MLP $\psi(\beta)$, which is directly added to the standard time embedding $\phi(t)$. This explicitly conditions the learned vector field $v_\theta(x_t, t, \beta)$ on the regularization parameter. 

Intuitively, varying $\beta$ parameterizes a family of transport trajectories between the prior and the data. This allows the model to learn multiple routing strategies, ranging from straight, highly regularized paths to complex, entropic ones.

At inference time, we sample directly from a standard Gaussian prior $x_0 \sim \mathcal{N}(0, \mathbf{I})$ and supply a fixed $\beta$ to guide the generation path. Table \ref{tab:beta_results} presents the CIFAR-10 results. Tuning $\beta$ at inference shifts the performance profile across different step budgets, allowing users to dynamically balance generation quality and path alignment post-hoc without retraining the model.

\subsection{Qualitative Analysis of Prior-Data Alignment}
\label{sec:appendix_alignment_vis}

Figure \ref{fig:noise_alignment_appendix} provides a visual demonstration of noise-data alignment. While standard Gaussian noise (middle row) shares near-zero spatial correlation with the target image, our aligned noise achieves a much higher cosine similarity with the target data while preserving its Gaussian characteristics.

\begin{figure}[!htbp]
    \centering
    \setlength{\tabcolsep}{1.5pt}
    
    \scalebox{0.95}{
    \begin{tabular}{c cccccccc}
        
        \rotatebox{90}{\makebox[1.5cm][c]{\small \textbf{Original}}} &
        \includegraphics[width=0.11\linewidth]{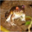} &
        \includegraphics[width=0.11\linewidth]{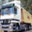} &
        \includegraphics[width=0.11\linewidth]{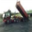} &
        \includegraphics[width=0.11\linewidth]{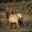} &
        \includegraphics[width=0.11\linewidth]{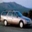} &
        \includegraphics[width=0.11\linewidth]{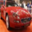} &
        \includegraphics[width=0.11\linewidth]{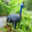} &
        \includegraphics[width=0.11\linewidth]{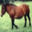} \\
        
        \noalign{\vskip 15pt} 
        
        \rotatebox{90}{\makebox[1.5cm][c]{\small \textbf{Initial Noise}}} &
        \includegraphics[width=0.11\linewidth]{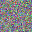} &
        \includegraphics[width=0.11\linewidth]{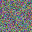} &
        \includegraphics[width=0.11\linewidth]{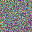} &
        \includegraphics[width=0.11\linewidth]{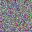} &
        \includegraphics[width=0.11\linewidth]{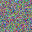} &
        \includegraphics[width=0.11\linewidth]{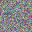} &
        \includegraphics[width=0.11\linewidth]{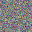} &
        \includegraphics[width=0.11\linewidth]{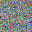} \\
        
        & \scriptsize $r = -0.023$ & \scriptsize $r = -0.019$ & \scriptsize $r = 0.003$ & \scriptsize $r = -0.000$ & \scriptsize $r = 0.012$ & \scriptsize $r = -0.018$ & \scriptsize $r = -0.006$ & \scriptsize $r = -0.014$ \\

        \noalign{\vskip 2pt}
        
        \rotatebox{90}{\makebox[1.5cm][c]{\small \textbf{Aligned}}} &
        \includegraphics[width=0.11\linewidth]{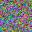} &
        \includegraphics[width=0.11\linewidth]{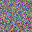} &
        \includegraphics[width=0.11\linewidth]{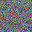} &
        \includegraphics[width=0.11\linewidth]{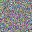} &
        \includegraphics[width=0.11\linewidth]{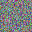} &
        \includegraphics[width=0.11\linewidth]{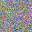} &
        \includegraphics[width=0.11\linewidth]{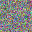} &
        \includegraphics[width=0.11\linewidth]{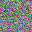} \\

        & \scriptsize $r = 0.072$ & \scriptsize $r = 0.085$ & \scriptsize $r = 0.061$ & \scriptsize $r = 0.048$ & \scriptsize $r = 0.048$ & \scriptsize $r = 0.046$ & \scriptsize $r = 0.059$ & \scriptsize $r = 0.073$ \\
        
    \end{tabular}
    }
    
    \vspace{0.5em}
    \caption{Visual comparison of original CIFAR-10 data (top), standard initial Gaussian noise (middle), and the corresponding optimized noise (bottom). The \textit{cosine similarity $r$} between the original image and the noise is reported below each sample.}
    \label{fig:noise_alignment_appendix}
\end{figure}




\label{sec:latent_space_generation}

\section{Few-Step Sampling Quality}
\label{sec:few_step_quality}

To illustrate how CNA performs in the few-step regime, Figure \ref{fig:appendix_visual_comparison} compares qualitative results across different sampling budgets. As seen in the figure, CNA consistently produces higher quality samples than the baselines, especially when constrained to just 2 or 4 steps.

\section{Additional Qualitative Results}
\label{app:E}
We provide additional qualitative examples that are generated using CNA. All images are produced with the Euler solver at 64 sampling steps. Results for CIFAR-10 and ImageNet32 are shown in Figures \ref{fig:cifar10_10x10} and \ref{fig:imagenet32_10x10}, respectively.

\section{Use of Large Language Models}
\label{app:F}
We used Large Language Models (LLMs) solely for the refinement and polishing of this paper. The core research, including the methodological formulation, experimental design, and empirical analysis, remains the exclusive work of the authors.

\begin{figure}[!htbp]
    \centering
    \small

    \newcommand{\drawsample}[1]{
        \setlength{\tabcolsep}{1.2pt}
        \adjustbox{max height=0.8\textheight, max width=0.7\textwidth}{
        \begin{tabular}{ccccc}
            & \small \textbf{1 NFE} & \small \textbf{2 NFE} & \small \textbf{4 NFE} & \small \textbf{8 NFE} \\

            \rotatebox{90}{\makebox[1.1cm][c]{\small \textbf{I-CFM}}} &
            \includegraphics[width=0.18\linewidth]{figures/Baseline/fake_#1_1NFE.png} &
            \includegraphics[width=0.18\linewidth]{figures/Baseline/fake_#1_2NFE.png} &
            \includegraphics[width=0.18\linewidth]{figures/Baseline/fake_#1_4NFE.png} &
            \includegraphics[width=0.18\linewidth]{figures/Baseline/fake_#1_8NFE.png} \\

            \noalign{\vskip 2pt}

            \rotatebox{90}{\makebox[1.1cm][c]{\small \textbf{OT-CFM}}} &
            \includegraphics[width=0.18\linewidth]{figures/BatchOT/fake_#1_1NFE.png} &
            \includegraphics[width=0.18\linewidth]{figures/BatchOT/fake_#1_2NFE.png} &
            \includegraphics[width=0.18\linewidth]{figures/BatchOT/fake_#1_4NFE.png} &
            \includegraphics[width=0.18\linewidth]{figures/BatchOT/fake_#1_8NFE.png} \\

            \noalign{\vskip 2pt}

            \rotatebox{90}{\makebox[1.1cm][c]{\small \textbf{Ours}}} &
            \includegraphics[width=0.18\linewidth]{figures/Ours/fake_#1_1NFE.png} &
            \includegraphics[width=0.18\linewidth]{figures/Ours/fake_#1_2NFE.png} &
            \includegraphics[width=0.18\linewidth]{figures/Ours/fake_#1_4NFE.png} &
            \includegraphics[width=0.18\linewidth]{figures/Ours/fake_#1_8NFE.png} \\
        \end{tabular}
        }
        \vspace{4mm} 
    }

    \begin{subfigure}[t]{0.48\textwidth}
        \centering
        \drawsample{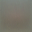}
        \drawsample{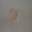}
        \drawsample{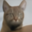}
        \drawsample{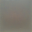}
        \caption{CIFAR-10}
    \end{subfigure}
    \hfill
    \begin{subfigure}[t]{0.48\textwidth}
        \centering
        \drawsample{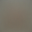}
        \drawsample{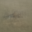}
        \drawsample{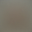}
        \drawsample{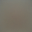}
        \caption{ImageNet32}
    \end{subfigure}

    \caption{\textbf{Qualitative comparison across different step counts.} Comparison of sampling trajectories for CIFAR-10 (left) and ImageNet32 (right). At 1 and 2 NFEs, our method maintains significantly better global structure than the baseline variants.}
    \label{fig:appendix_visual_comparison}
\end{figure}

\begin{figure}[ht]
    \centering
    \offinterlineskip 
    
    \includegraphics[width=0.1\linewidth]{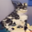}%
    \includegraphics[width=0.1\linewidth]{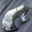}%
    \includegraphics[width=0.1\linewidth]{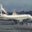}%
    \includegraphics[width=0.1\linewidth]{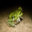}%
    \includegraphics[width=0.1\linewidth]{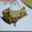}%
    \includegraphics[width=0.1\linewidth]{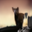}%
    \includegraphics[width=0.1\linewidth]{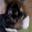}%
    \includegraphics[width=0.1\linewidth]{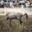}%
    \includegraphics[width=0.1\linewidth]{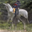}%
    \includegraphics[width=0.1\linewidth]{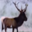}\\
    
    \includegraphics[width=0.1\linewidth]{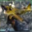}%
    \includegraphics[width=0.1\linewidth]{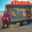}%
    \includegraphics[width=0.1\linewidth]{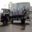}%
    \includegraphics[width=0.1\linewidth]{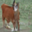}%
    \includegraphics[width=0.1\linewidth]{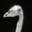}%
    \includegraphics[width=0.1\linewidth]{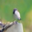}%
    \includegraphics[width=0.1\linewidth]{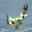}%
    \includegraphics[width=0.1\linewidth]{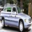}%
    \includegraphics[width=0.1\linewidth]{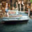}%
    \includegraphics[width=0.1\linewidth]{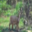}\\
    
    \includegraphics[width=0.1\linewidth]{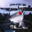}%
    \includegraphics[width=0.1\linewidth]{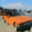}%
    \includegraphics[width=0.1\linewidth]{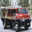}%
    \includegraphics[width=0.1\linewidth]{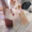}%
    \includegraphics[width=0.1\linewidth]{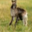}%
    \includegraphics[width=0.1\linewidth]{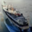}%
    \includegraphics[width=0.1\linewidth]{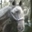}%
    \includegraphics[width=0.1\linewidth]{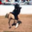}%
    \includegraphics[width=0.1\linewidth]{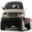}%
    \includegraphics[width=0.1\linewidth]{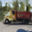}\\
    
    \includegraphics[width=0.1\linewidth]{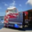}%
    \includegraphics[width=0.1\linewidth]{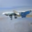}%
    \includegraphics[width=0.1\linewidth]{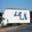}%
    \includegraphics[width=0.1\linewidth]{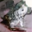}%
    \includegraphics[width=0.1\linewidth]{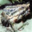}%
    \includegraphics[width=0.1\linewidth]{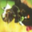}%
    \includegraphics[width=0.1\linewidth]{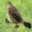}%
    \includegraphics[width=0.1\linewidth]{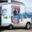}%
    \includegraphics[width=0.1\linewidth]{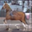}%
    \includegraphics[width=0.1\linewidth]{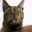}\\
    
    \includegraphics[width=0.1\linewidth]{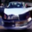}%
    \includegraphics[width=0.1\linewidth]{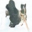}%
    \includegraphics[width=0.1\linewidth]{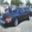}%
    \includegraphics[width=0.1\linewidth]{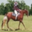}%
    \includegraphics[width=0.1\linewidth]{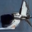}%
    \includegraphics[width=0.1\linewidth]{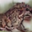}%
    \includegraphics[width=0.1\linewidth]{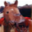}%
    \includegraphics[width=0.1\linewidth]{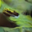}%
    \includegraphics[width=0.1\linewidth]{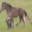}%
    \includegraphics[width=0.1\linewidth]{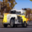}\\
    
    \includegraphics[width=0.1\linewidth]{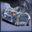}%
    \includegraphics[width=0.1\linewidth]{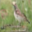}%
    \includegraphics[width=0.1\linewidth]{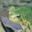}%
    \includegraphics[width=0.1\linewidth]{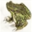}%
    \includegraphics[width=0.1\linewidth]{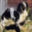}%
    \includegraphics[width=0.1\linewidth]{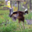}%
    \includegraphics[width=0.1\linewidth]{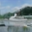}%
    \includegraphics[width=0.1\linewidth]{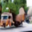}%
    \includegraphics[width=0.1\linewidth]{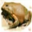}%
    \includegraphics[width=0.1\linewidth]{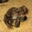}\\
    
    \includegraphics[width=0.1\linewidth]{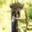}%
    \includegraphics[width=0.1\linewidth]{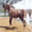}%
    \includegraphics[width=0.1\linewidth]{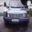}%
    \includegraphics[width=0.1\linewidth]{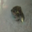}%
    \includegraphics[width=0.1\linewidth]{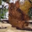}%
    \includegraphics[width=0.1\linewidth]{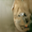}%
    \includegraphics[width=0.1\linewidth]{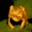}%
    \includegraphics[width=0.1\linewidth]{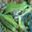}%
    \includegraphics[width=0.1\linewidth]{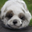}%
    \includegraphics[width=0.1\linewidth]{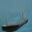}\\
    
    \includegraphics[width=0.1\linewidth]{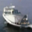}%
    \includegraphics[width=0.1\linewidth]{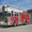}%
    \includegraphics[width=0.1\linewidth]{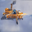}%
    \includegraphics[width=0.1\linewidth]{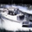}%
    \includegraphics[width=0.1\linewidth]{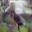}%
    \includegraphics[width=0.1\linewidth]{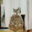}%
    \includegraphics[width=0.1\linewidth]{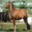}%
    \includegraphics[width=0.1\linewidth]{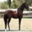}%
    \includegraphics[width=0.1\linewidth]{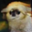}%
    \includegraphics[width=0.1\linewidth]{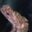}\\
    
    \includegraphics[width=0.1\linewidth]{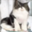}%
    \includegraphics[width=0.1\linewidth]{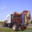}%
    \includegraphics[width=0.1\linewidth]{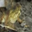}%
    \includegraphics[width=0.1\linewidth]{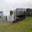}%
    \includegraphics[width=0.1\linewidth]{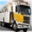}%
    \includegraphics[width=0.1\linewidth]{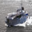}%
    \includegraphics[width=0.1\linewidth]{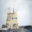}%
    \includegraphics[width=0.1\linewidth]{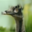}%
    \includegraphics[width=0.1\linewidth]{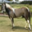}%
    \includegraphics[width=0.1\linewidth]{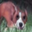}\\
    
    \includegraphics[width=0.1\linewidth]{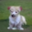}%
    \includegraphics[width=0.1\linewidth]{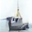}%
    \includegraphics[width=0.1\linewidth]{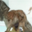}%
    \includegraphics[width=0.1\linewidth]{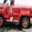}%
    \includegraphics[width=0.1\linewidth]{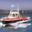}%
    \includegraphics[width=0.1\linewidth]{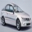}%
    \includegraphics[width=0.1\linewidth]{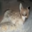}%
    \includegraphics[width=0.1\linewidth]{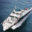}%
    \includegraphics[width=0.1\linewidth]{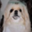}%
    \includegraphics[width=0.1\linewidth]{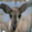}
    
    \caption{\textbf{Uncurated 32$\times$32 samples on CIFAR-10} ($\beta = 5$).}
    \label{fig:cifar10_10x10}
\end{figure}

\begin{figure}[ht]
    \centering
    \offinterlineskip 
    
    \includegraphics[width=0.1\linewidth]{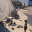}%
    \includegraphics[width=0.1\linewidth]{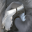}%
    \includegraphics[width=0.1\linewidth]{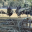}%
    \includegraphics[width=0.1\linewidth]{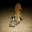}%
    \includegraphics[width=0.1\linewidth]{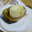}%
    \includegraphics[width=0.1\linewidth]{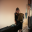}%
    \includegraphics[width=0.1\linewidth]{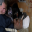}%
    \includegraphics[width=0.1\linewidth]{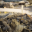}%
    \includegraphics[width=0.1\linewidth]{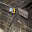}%
    \includegraphics[width=0.1\linewidth]{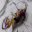}\\
    
    \includegraphics[width=0.1\linewidth]{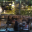}%
    \includegraphics[width=0.1\linewidth]{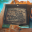}%
    \includegraphics[width=0.1\linewidth]{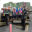}%
    \includegraphics[width=0.1\linewidth]{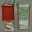}%
    \includegraphics[width=0.1\linewidth]{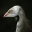}%
    \includegraphics[width=0.1\linewidth]{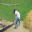}%
    \includegraphics[width=0.1\linewidth]{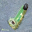}%
    \includegraphics[width=0.1\linewidth]{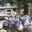}%
    \includegraphics[width=0.1\linewidth]{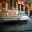}%
    \includegraphics[width=0.1\linewidth]{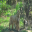}\\
    
    \includegraphics[width=0.1\linewidth]{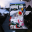}%
    \includegraphics[width=0.1\linewidth]{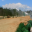}%
    \includegraphics[width=0.1\linewidth]{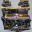}%
    \includegraphics[width=0.1\linewidth]{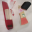}%
    \includegraphics[width=0.1\linewidth]{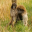}%
    \includegraphics[width=0.1\linewidth]{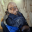}%
    \includegraphics[width=0.1\linewidth]{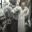}%
    \includegraphics[width=0.1\linewidth]{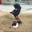}%
    \includegraphics[width=0.1\linewidth]{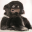}%
    \includegraphics[width=0.1\linewidth]{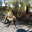}\\
    
    \includegraphics[width=0.1\linewidth]{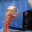}%
    \includegraphics[width=0.1\linewidth]{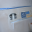}%
    \includegraphics[width=0.1\linewidth]{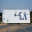}%
    \includegraphics[width=0.1\linewidth]{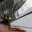}%
    \includegraphics[width=0.1\linewidth]{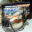}%
    \includegraphics[width=0.1\linewidth]{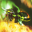}%
    \includegraphics[width=0.1\linewidth]{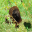}%
    \includegraphics[width=0.1\linewidth]{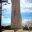}%
    \includegraphics[width=0.1\linewidth]{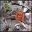}%
    \includegraphics[width=0.1\linewidth]{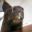}\\
    
    \includegraphics[width=0.1\linewidth]{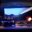}%
    \includegraphics[width=0.1\linewidth]{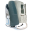}%
    \includegraphics[width=0.1\linewidth]{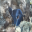}%
    \includegraphics[width=0.1\linewidth]{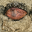}%
    \includegraphics[width=0.1\linewidth]{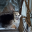}%
    \includegraphics[width=0.1\linewidth]{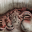}%
    \includegraphics[width=0.1\linewidth]{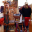}%
    \includegraphics[width=0.1\linewidth]{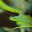}%
    \includegraphics[width=0.1\linewidth]{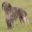}%
    \includegraphics[width=0.1\linewidth]{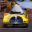}\\
    
    \includegraphics[width=0.1\linewidth]{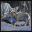}%
    \includegraphics[width=0.1\linewidth]{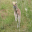}%
    \includegraphics[width=0.1\linewidth]{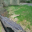}%
    \includegraphics[width=0.1\linewidth]{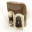}%
    \includegraphics[width=0.1\linewidth]{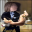}%
    \includegraphics[width=0.1\linewidth]{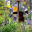}%
    \includegraphics[width=0.1\linewidth]{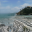}%
    \includegraphics[width=0.1\linewidth]{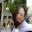}%
    \includegraphics[width=0.1\linewidth]{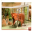}%
    \includegraphics[width=0.1\linewidth]{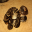}\\
    
    \includegraphics[width=0.1\linewidth]{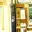}%
    \includegraphics[width=0.1\linewidth]{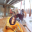}%
    \includegraphics[width=0.1\linewidth]{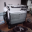}%
    \includegraphics[width=0.1\linewidth]{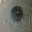}%
    \includegraphics[width=0.1\linewidth]{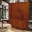}%
    \includegraphics[width=0.1\linewidth]{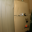}%
    \includegraphics[width=0.1\linewidth]{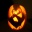}%
    \includegraphics[width=0.1\linewidth]{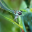}%
    \includegraphics[width=0.1\linewidth]{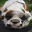}%
    \includegraphics[width=0.1\linewidth]{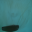}\\
    
    \includegraphics[width=0.1\linewidth]{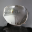}%
    \includegraphics[width=0.1\linewidth]{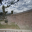}%
    \includegraphics[width=0.1\linewidth]{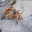}%
    \includegraphics[width=0.1\linewidth]{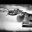}%
    \includegraphics[width=0.1\linewidth]{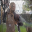}%
    \includegraphics[width=0.1\linewidth]{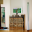}%
    \includegraphics[width=0.1\linewidth]{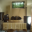}%
    \includegraphics[width=0.1\linewidth]{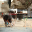}%
    \includegraphics[width=0.1\linewidth]{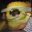}%
    \includegraphics[width=0.1\linewidth]{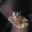}\\
    
    \includegraphics[width=0.1\linewidth]{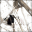}%
    \includegraphics[width=0.1\linewidth]{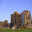}%
    \includegraphics[width=0.1\linewidth]{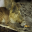}%
    \includegraphics[width=0.1\linewidth]{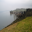}%
    \includegraphics[width=0.1\linewidth]{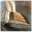}%
    \includegraphics[width=0.1\linewidth]{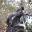}%
    \includegraphics[width=0.1\linewidth]{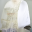}%
    \includegraphics[width=0.1\linewidth]{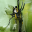}%
    \includegraphics[width=0.1\linewidth]{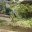}%
    \includegraphics[width=0.1\linewidth]{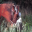}\\
    
    \includegraphics[width=0.1\linewidth]{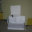}%
    \includegraphics[width=0.1\linewidth]{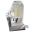}%
    \includegraphics[width=0.1\linewidth]{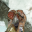}%
    \includegraphics[width=0.1\linewidth]{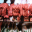}%
    \includegraphics[width=0.1\linewidth]{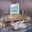}%
    \includegraphics[width=0.1\linewidth]{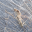}%
    \includegraphics[width=0.1\linewidth]{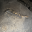}%
    \includegraphics[width=0.1\linewidth]{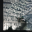}%
    \includegraphics[width=0.1\linewidth]{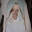}%
    \includegraphics[width=0.1\linewidth]{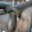}
    
    \caption{\textbf{Uncurated 32$\times$32 samples on ImageNet32} ($\beta = 30$)}
    \label{fig:imagenet32_10x10}
\end{figure}

\end{document}